\documentclass{article}

\usepackage[preprint]{neurips_2026}

\usepackage[utf8]{inputenc} 
\usepackage[T1]{fontenc}    
\usepackage{placeins} 
\usepackage{hyperref}       
\usepackage{url}            
\usepackage{booktabs}       
\usepackage{amsfonts}       
\usepackage{nicefrac}       
\usepackage{microtype}      
\usepackage{xcolor}         
\usepackage{amsmath}
\usepackage{amssymb}
\usepackage{amsthm}
\usepackage{bm}
\usepackage{algorithm}
\usepackage{algorithmic}
\usepackage{graphicx}
\usepackage{comment}
\usepackage{enumitem}
\usepackage{mathtools}
\usepackage[table]{xcolor}
\usepackage[most]{tcolorbox} 
\usepackage{tikz}
\usetikzlibrary{shapes,arrows,positioning,fit,backgrounds}
\tikzstyle{block} = [rectangle, draw, fill=blue!20,
    text width=8em, text centered, rounded corners, minimum height=3em]
\tikzstyle{cloud} = [draw, ellipse, fill=green!20,
    min width=8em, text width=8em, text centered]
\tikzstyle{line} = [draw, -latex']
\tikzstyle{container} = [draw, rectangle, dashed, inner sep=0.3cm]

\newtheorem{theorem}{Theorem}

\newtheorem{lemma}{Lemma}
\newtheorem{remark}{Remark}
\newtheorem{definition}{Definition}
\newtheorem{assumption}{Assumption}
\newtheorem*{theorem*}{Theorem}
\DeclareMathOperator{\clip}{clip}

\title{H\"older Policy Optimisation}
 
\renewcommand{\footnotemark}{}
 \vspace{-3mm}

\author{%
\small
\parbox{\textwidth}{\centering\bfseries
  \mbox{Yuxiang Chen\textsuperscript{1,*}\thanks{%
    \textsuperscript{*}Equal contribution.\quad
    \textsuperscript{\textdagger}Corresponding author: \texttt{jun.wang@cs.ucl.ac.uk}.\quad
    Code available at \url{https://github.com/YihangChen9/HolderPO}.
    \textsuperscript{1}University College London, London, United Kingdom.\quad
    \textsuperscript{2}Shanghai Jiao Tong University, Shanghai, China.\quad
    \textsuperscript{3}The Hong Kong University of Science and Technology (Guangzhou), Guangzhou, China.%
  }}\quad
  \mbox{Dingli Liang\textsuperscript{1,*}}\quad
  \mbox{Yihang Chen\textsuperscript{1,*}}\quad
  \mbox{Ziqin Gong\textsuperscript{3}}\quad
  \mbox{Chenyang Le\textsuperscript{2}}\quad
  \mbox{Zhaokai Wang\textsuperscript{2}}%
}\\[2pt]
\small
\parbox{\textwidth}{\centering\bfseries
  \mbox{Jiachen Zhu\textsuperscript{2}}\quad
  \mbox{Lingyu Yang\textsuperscript{2}}\quad
  \mbox{Jianghao Lin\textsuperscript{2}}\quad
  \mbox{Weinan Zhang\textsuperscript{2}}\quad
  \mbox{Jun Wang\textsuperscript{1,\textdagger}}%
}%
}

\begin{document}

\maketitle

\vspace{-5mm}
\begin{abstract}
Group Relative Policy Optimisation (GRPO) enhances large language models by estimating advantages across a group of sampled trajectories. However, mapping these trajectory-level advantages to policy updates requires aggregating token-level probabilities within each sequence. Relying on a fixed aggregation mechanism for this step fundamentally limits the algorithm's adaptability. Empirically, we observe a critical trade-off: certain fixed aggregations frequently suffer from training collapse, while others fail to yield satisfactory performance. To resolve this, we propose \textbf{H\"{o}lderPO}, a generalised policy optimisation framework unifying token-level probability aggregation via the H\"{o}lder mean. By explicitly modulating the parameter $p$, our framework provides continuous control over the trade-off between gradient concentration and variance bounds. Theoretically, we prove that a larger $p$ concentrates the gradient to amplify sparse learning signals, whereas a smaller $p$ strictly bounds gradient variance. Because no static configuration can universally resolve this concentration-stability trade-off, we instantiate the framework with a dynamic annealing algorithm that progressively schedules $p$ across the training lifecycle. Extensive evaluations demonstrate superior stability and convergence over existing baselines. Specifically, our approach achieves a state-of-the-art average accuracy of $54.9\%$ across multiple mathematical benchmarks, yielding a substantial $7.2\%$ relative gain over standard GRPO and secures an exceptional $93.8\%$ success rate on ALFWorld.
 
\end{abstract}

\vspace{-4mm}
\begin{figure}[htbp]
    \centering
    \includegraphics[width=1\linewidth]{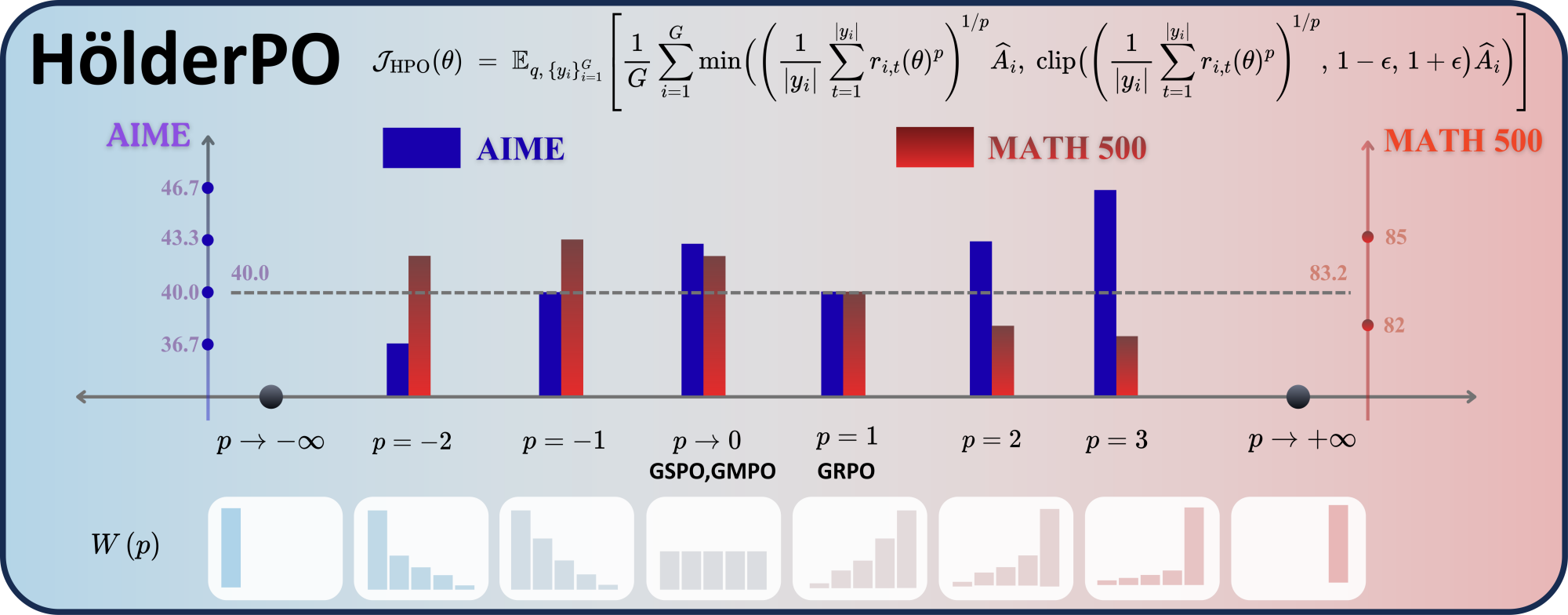}
    \vspace{-5mm}
    \caption{\textbf{H\"{o}lderPO unifies token-level aggregation under a single parameter $p$.} The objective at the top generalises GRPO by replacing its arithmetic mean over token-level importance ratios with the H\"{o}lder mean of order $p \in \mathbb{R}$, recovering GRPO ($p=1$) and GMPO/GSPO ($p \to 0$) as special cases. The bar chart reports accuracy on AIME24 (blue, sparse signal) and MATH500 (red, dense signal), with dashed lines marking GRPO baselines. Bottom: the token weight distribution $W(p)$, with each panel ordering tokens from small (left) to large (right) importance ratio.}
    \label{fig:placeholder}
\end{figure}

\section{Introduction}
Reinforcement Learning (RL) has emerged as a key technique for advancing the alignment and complex reasoning capabilities of Large Language Models (LLMs)~\citep{ouyang2022training,schulman2017proximal}. Recently, Group Relative Policy Optimisation (GRPO) has emerged as a highly effective and compute-efficient algorithm, largely driving the success of reasoning models like DeepSeek-R1~\citep{shao2024deepseekmath}. GRPO operates by estimating advantages across a group of sampled trajectories, substantially reducing training overhead by eliminating the need for an external critic model. However, mapping these trajectory-level advantages to policy updates requires aggregating token-level probabilities within each sequence. As the demand for solving long-horizon reasoning tasks grows, the fundamental mechanics of this fixed aggregation step have come under scrutiny~\citep{liu2025understanding}. Existing algorithms rigidly rely on static aggregation functions: standard GRPO ($p = 1$) defaults to the Arithmetic Mean, while recent variants like GMPO~\citep{zhao2025geometric} and GSPO~\citep{zheng2025group} ($p \rightarrow 0$) attempt to mitigate variance by employing the Geometric Mean.

Despite their empirical success, these fixed aggregation mechanisms implicitly impose a static optimisation landscape, limiting their adaptability across long-horizon reasoning tasks of varying signal density --- the regime in which the trade-off we identify becomes acute. Through empirical investigation, we observe a critical trade-off: certain fixed aggregations frequently suffer from training collapse, while others fail to yield satisfactory performance. Specifically, on \textbf{dense-signal tasks} (where supervision is distributed across many tokens, e.g., MATH~\citep{hendrycks2021measuring}), standard GRPO ($p = 1$) disproportionately over-weights minor token-level errors, inducing high-variance gradient updates that can lead to training collapse. Conversely, on \textbf{sparse-signal tasks} (where correct reasoning is concentrated in rare, high-magnitude tokens, e.g., AIME~\citep{numina_math_datasets}), GSPO ($p \rightarrow 0$) overly smooths the probability ratios, suppressing the effective use of these rare ``aha moments''. Figure~\ref{fig:placeholder} 
visualises this divergence: AIME24 accuracy peaks at $p=3$ while MATH500 
peaks at $p=-1$, with the bottom row showing how the underlying token 
weight distribution $W(p)$ deforms across the $p$-axis. Essentially, 
there is no ``silver bullet'' among static mean functions; the optimal 
probability aggregation is not a constant, but rather a function of 
task signal density and the model's training progression.

To address these fundamental limitations, we propose \textbf{H\"{o}lderPO}, a generalised policy optimisation framework unifying token-level probability aggregation via the adaptable H\"{o}lder mean ($p$-norm). By explicitly modulating the parameter $p$, the framework provides continuous control over the trade-off between gradient concentration and variance bounds. Theoretically, we prove a two-sided trade-off in $p$: a larger $p$ concentrates the gradient weight distribution on a small subset of tokens, amplifying the effective use of rare informative learning signals at the cost of looser variance bounds. Conversely, a smaller $p$ strictly tightens the variance of the policy gradient estimator, ensuring training stability at the cost of weakening the response to those same sparse signals. Because no static configuration can simultaneously realise both endpoint advantages, we instantiate the H\"{o}lderPO framework with a dynamic annealing algorithm. By progressively scheduling $p$ from a higher positive value to a negative value during training, this algorithm seamlessly transitions the model from aggressive signal amplification in the early stages to variance-controlled convergence in the later stages.

Extensive empirical evaluations across a comprehensive suite of complex reasoning and decision-making benchmarks strongly validate our claims. Built upon the \textbf{Qwen2.5-Math-7B} base~\citep{yang2024qwen25mathtechnicalreportmathematical}, our ablation studies first confirm the task-specific sensitivity of $p$: sparse-signal tasks strictly favour higher $p$ values for aggressive signal amplification, whereas dense-signal tasks benefit from lower (possibly negative) $p$ values for gradient stability. Crucially, when explicitly setting $p=3$, our approach effectively breaks the existing performance ceiling on the highly challenging AIME benchmark, surpassing the previous $43.3\%$ accuracy record to achieve $46.7\%$. Building on these insights, by employing our dynamic annealing algorithm, H\"olderPO unifies these advantages without incurring additional computational overhead. Consequently, our approach achieves a state-of-the-art average accuracy of \textbf{54.9\%} across five mathematical benchmarks (AIME, AMC, MATH, Minerva~\citep{lewkowycz2022solving}, and OlympiadBench~\citep{he-etal-2024-olympiadbench}), a $7.2\%$ relative gain over standard GRPO that surpasses concurrent token-aggregation methods including PMPO~\citep{zhao2026one}. Beyond mathematical reasoning, this dynamic adaptability extends to open-world agentic tasks, securing an exceptional \textbf{93.8\%} success rate on the ALFWorld benchmark~\citep{shridhar2020alfworld}, a $28.8\%$ relative gain over GRPO ($72.8\%$).

In summary, our main contributions are as follows:
\begin{itemize}
    \item \textbf{The H\"{o}lderPO Framework:} We propose H\"{o}lderPO, a generalised policy optimisation framework that dynamically unifies various mean-based probability aggregations through the adaptable H\"{o}lder parameter $p$.
    
    \item \textbf{Theoretical Foundation:} We theoretically characterise the two-sided role of $p$ in long-horizon reasoning: a larger $p$ concentrates gradient weight to amplify sparse learning signals, whereas a smaller $p$ strictly bounds gradient variance to ensure training stability. No fixed $p$ realises both endpoint advantages simultaneously, motivating dynamic scheduling.
    
    \item \textbf{Empirical Breakthroughs and SOTA Performance:} Empirically, explicitly employing a large $p=3$ breaks the existing performance ceiling on the highly challenging AIME benchmark. Furthermore, instantiating the framework with a dynamic $p$-annealing algorithm achieves state-of-the-art results, securing a $54.9\%$ average accuracy across five mathematical benchmarks and an exceptional $93.8\%$ success rate on ALFWorld agentic tasks.
\end{itemize}

\section{Related Work}
\label{sec:related_work}
\textbf{Reinforcement Learning for Complex Reasoning.}
Reinforcement Learning (RL) has become the cornerstone of LLM post-training. While foundational work used RLHF for behavioural alignment~\citep{ouyang2022training, stiennon2020learning}, recent advances focus on complex reasoning via RLVR~\citep{wen2025reinforcement}, pioneered by OpenAI o-series~\citep{jaech2024openai} and DeepSeek-R1~\citep{guo2025deepseek, shao2024deepseekmath}, inspiring both proprietary~\citep{comanici2025gemini, yang2025qwen3} and open-source successors. GRPO~\citep{shao2024deepseekmath} has emerged as the dominant algorithm; its broader ecosystem of refinements is surveyed in Appendix~\ref{app:extended_related_work}. 

\textbf{Token-Level Aggregation.}
The aggregation operator that maps token-level importance ratios to a sequence-level signal is the most direct analogue of our framework. GRPO uses the arithmetic mean, while GMPO~\citep{zhao2025geometric} and GSPO~\citep{zheng2025group} adopt the geometric mean to mitigate outlier variance. Concurrent PMPO~\citep{zhao2026one} parameterises a power-mean exponent $p\in[0,1]$, adapted per-trajectory via clip-aware ESS matching. Our framework differs in two key respects: (i) we extend $p$ to the \emph{full real range}, identifying $p<0$ as a qualitatively distinct \emph{inverse-concentration} phase unexplored by prior work; and (ii) we adapt $p$ along the \emph{temporal} axis (across training steps) rather than per trajectory, enabling complementary roles for early-stage signal amplification and late-stage variance contraction.

\textbf{Token Reweighting via Auxiliary Signals.}
A parallel line reweights tokens within each rollout using signals \emph{external} to the importance ratio: token entropy~\citep{wang2025beyond, yu2026erpo, simoni2025gtpo}, token probability~\citep{yang2025not}, hidden contributions to response confidence~\citep{deng2025token}, or selective KL masking~\citep{lin2025token}. These approaches are orthogonal to ours and could in principle be combined with H\"{o}lderPO's power-mean aggregation.

\section{H\"{o}lderPO: A Generalised Aggregation Framework}
\label{sec:method}

When adapting PPO for LLMs, particularly for training long-horizon reasoning tasks,  group-based variants like GRPO~\citep{shao2024deepseekmath} formulate the unclipped objective as
$$\mathcal{J}(\theta) = \mathbb{E}_{x, \{y_i\}} \left[ \frac{1}{G} \sum_{i=1}^G \rho_i(\theta) \widehat{A}_i \right].$$
Here, $\rho_i(\theta)$ is the sequence-level surrogate term, which can be regarded as an \emph{aggregation operator}--- a functional projection that compresses the full sequence of token-level importance ratios $\{r_{i,t}(\theta)\}_{t=1}^{|y_i|}$ into a well-behaved sequence-level scalar. While GRPO uses the arithmetic mean, GMPO~\citep{zhao2025geometric} and GSPO~\citep{zheng2025group} use geometric mean. However, these methods only represent static, isolated points within a broader, continuous spectrum of aggregation operators. 

In section~\ref{sec:framework}, we propose \textbf{H\"{o}lder Policy Optimisation}, a generalised framework that parameterises the aggregation operators by a single scalar $p \in \mathbb{R}$ via the Hölder mean. Pivotally, the single parameter $p$ governs a trade-off between \textbf{gradient concentration} (defined in Section~\ref{sec:concentration}), which selectively amplifies targeted learning signals, and the \textbf{variance bound} (analysed in Section~\ref{sec:variance}), which ensures training stability. Finally, the interplay between these two competing properties motivates our \textbf{dynamic scheduling strategy} in Section~\ref{sec:dynamic}.


\subsection{Aggregation via the H\"{o}lder Mean}
\label{sec:framework}

Given a prompt context $x$ and a rollout $y_i$ sampled from $\pi_{\theta_{\text{old}}}$, the token-level importance ratio for $t$-th token is $r_{i,t}(\theta) \;=\; \frac{\pi_\theta(y_{i,t} \mid x, y_{i,<t})}{\pi_{\theta_{\text{old}}}(y_{i,t} \mid x, y_{i,<t})}$. Rather than relying on a fixed operator, H\"{o}lderPO generalises the token-level aggregation by the H\"{o}lder mean of order $p$:
\begin{equation}
\label{eq:rho-def}
    \rho_{i,p}(\theta) \;=\; 
    \begin{cases} 
        \left( \frac{1}{|y_i|} \sum_{t=1}^{|y_i|} r_{i,t}(\theta)^p \right)^{\!1/p}, & \text{if } p \neq 0, \\[2ex]
        \exp\!\left( \frac{1}{|y_i|} \sum_{t=1}^{|y_i|} \log r_{i,t}(\theta) \right), & \text{if } p = 0.
    \end{cases}
\end{equation}
Due to the limit for $p\rightarrow 0$, we take the geometric mean for $p=0$ branch (see Appendix~\ref{F4}). The H\"{o}lderPO objective then takes the standard PPO-style form with sequence-level clipping:
\begin{equation}
\label{eq:hpo-obj}
    \mathcal{J}_{H_s}(\theta) \;=\; \mathbb{E}_{x,\,\{y_i\}_{i=1}^G}\!\left[ \frac{1}{G} \sum_{i=1}^G \min\!\Big( \rho_{i,p}(\theta) \widehat{A}_i,\; \mathrm{clip}\big(\rho_{i,p}(\theta),\, 1-\epsilon,\, 1+\epsilon\big) \widehat{A}_i \Big) \right].
\end{equation}
Here $\widehat{A}_i$ is the advantage estimator and $\epsilon$ is the clipping threshold. The reason we choose sequence-level clipping is to control gradient variance (see Appendix \ref{app:token_clip} and \ref{I2}). Specifically, $p = 1$ recovers GRPO (Appendix~\ref{tlclipping}), while $p=0$ recovers GSPO (Appendix~\ref{slclipping}). To analyse how $p$ shapes the optimisation, we study $\nabla_\theta \rho_{i,p}(\theta)$, which governs the direction of the policy gradients (see Eq. (\ref{pgunclip}), (\ref{tlclippg}), (\ref{slclippg})). A direct calculation (Appendix~\ref{uncliformulas}) yields
\begin{equation}
\label{eq:rho-grad}
    \nabla_\theta \rho_{i,p}(\theta) \;=\; \rho_{i,p}(\theta) \sum_{t=1}^{|y_i|} W_{i,t}(p) \,\cdot\, \nabla_\theta \log \pi_\theta(y_{i,t} \mid x, y_{i,<t}) \quad W_{i,t}(p) \;\coloneqq\; \frac{r_{i,t}(\theta)^p}{\sum_{k=1}^{|y_i|} r_{i,k}(\theta)^p},
\end{equation}
where the per‑token gradient weights $W_{i,t}(p)$ form a probability distribution denoted by $W_i^p$. Crucially, varying $p$ does not alter the per-token log-gradient directions; instead, it solely reweights the directions and modulates the weight distribution.





\subsection{Distributional Deformation and Gradient Concentration}
\label{sec:concentration}
We formalise the gradient concentration by analysing $W_i^p$ through two complementary lenses. Locally, Theorem~\ref{thm:amplification} (Appendix~\ref{localpro}) shows an increasingly strict token-level weight allocation: as $p$ grows, maximal-ratio tokens monotonically dominate. Non-maximal ones may briefly gain weight before strictly decaying to zero once the rising threshold $\mu_i(p)$ surpasses their log-ratios. Globally, our next result (Appendix~\ref{globalpro}) captures the dispersion of the entire weight distribution by Shannon entropy.

\begin{theorem}
\label{thm:deformation}
Assume the sequence $y_i$ contains at least two tokens with distinct importance ratios. Then Shannon entropy of the weight distribution attains its global maximum at $p = 0$, where $W^0_i = \tfrac{1}{|y_i|} \mathrm{Unif}$, and strictly decreases as $|p|$ increases. Moreover, as $p \to \pm\infty$, $W_i^p$ concentrates uniformly on the subset $\mathcal{T}^{+} = \arg\max_t r_{i,t}(\theta)$ and $\mathcal{T}^{-} = \arg\min_t r_{i,t}(\theta)$, respectively.
\end{theorem}
Together, these dual perspectives formally characterise gradient concentration—the skewing of the weight distribution toward a specific subset of tokens. By governing the intensity and target of this skew, $p$ shapes the gradient contributions in three distinct regimes:

\textbf{{Upward Concentration ($p > 0$)}.} A positive $p$ drives the gradient concentration toward tokens with relatively high importance ratios. A prevailing view suggests that RL for reasoning  primarily acts to sharpen the pre-existing knowledge distribution of the base model (e.g., \cite{zhou2023lima, li2024rain, yue2025does}). Under this view, an importance ratio $>1$ serves as a confidence signal that, ideally, highlights the critical bottleneck tokens within reasoning steps. In long-horizon tasks, where such high-confidence tokens are sparse \citep{zelikman2022star, lightman2023lets, yao2023tree}, setting $p > 0$ explicitly amplifies their weight to prevent their gradients from being diluted.



\paragraph{Uniform Dispersion ($p \rightarrow 0$).} As $p$ decreases, the specific contributions of individual tokens are increasingly flatten out. At $p=0$,  every token contributes equally.

\paragraph{Downward Concentration ($p < 0$).}
A negative $p$ inverts the gradient allocation, aggressively upweighting tokens with importance ratios $< 1$, which signal current model's hesitation and pinpoint unconventional yet effective decision points in successful trajectories. Consequently, a moderately negative $p$  promotes reasoning diversity by forcing the model to consolidate alternative pathways. More details about the relation between our gradient concentration mechanism and exploration-exploitation trade-off can be found in Appendix \ref{app:exploration}.

\subsection{Policy Gradient Variance Bound}
\label{sec:variance}
Next, we analyse the variance of the policy gradient estimator induced by (\ref{eq:hpo-obj}). In long-horizon reasoning, while concentration enables the amplification of targeted signals, it risks magnifying gradient variance. The next theorem (proof is in Appendix~\ref{I2}) shows that such selectivity can destabilise convergence if left uncontrolled.

\begin{theorem}
\label{thm:variance}
Let $\widehat{\nabla}_\theta \mathcal{J}_{H_s}$ (Eq.~(\ref{seqlevelest})) denote the unbiased mini-batch estimator induced by (\ref{eq:hpo-obj}). Assume $\|\nabla_\theta \log \pi_\theta(y_{i,t} \mid x, y_{i,<t})\| \le M$ for all tokens within the batch, the variance admits the  bound
\begin{equation}
\label{eq:var-bound}
    \|\mathrm{Var}(\widehat{\nabla}_\theta \mathcal{J}_{H_s})\| \;\le\; \frac{M^2}{B}\, \mathbb{E}\!\left[ \widehat{A}_i^{\,2}\, \rho_{i,p}^2(\theta) \right],
\end{equation}
which is monotonically increasing in $p$ for all $p \in \mathbb{R}$, where $B$ is the batch size.
\end{theorem}

In addition, if we assume approximate orthogonality of gradients of tokens within sequences (Assumption \ref{orthogonal}), we prove the variance itself has a global minimum at some $p^* \le 0$. (Theorem \ref{variance_optimal}). 


\paragraph{Trade-off with concentration.}
Theorems~\ref{thm:deformation} and~\ref{thm:variance} highlight a structural trade-off controlled by the scalar $p$: driving $p$ upward isolates targeted pivotal signals, but incurs the cost of a looser variance bound. While shifting $p$ downward strictly tightens this bound, it dilutes these critical signals or redirects the concentration entirely.  In long-horizon reasoning, this trade-off becomes a bottleneck: we must amplify sparse signals without letting variance scale uncontrollably across the entire trajectory. Therefore, \textbf{no fixed $p$ can be uniformly optimal}, since the optimal balance between these two requirements varies depending on the specific task and training stage.

\subsection{A Dynamic $p$-Scheduling Strategy}
\label{sec:dynamic}

The trade-off above motivates a dynamic schedule for long-horizon reasoning tasks that monotonically decays $p$ from a positive initial value $p_{\text{high}}$ to a low (possibly negative) terminal value $p_{\text{low}}$ over the course of training: $p(0) = p_{\text{high}}, p(T) = p_{\text{low}}, \text{and} \ p(t_1) \ge p(t_2) \;\;\forall\; 0 \le t_1 < t_2 \le T.$ The early phase leverages positive concentration to amplify sparse, high-magnitude signals signals crucial for initial policy improvement. In the late phase, the schedule focuses on contracting the variance bound to guarantee stable convergence. Where $p_{\text{low}} < 0$, the algorithm utilises inverse concentration, moderately redirecting the gradient towards underemphasised tokens to foster reasoning diversity.

\begin{theorem}
\label{thm:dynamic}
Let $V(p) \coloneqq \mathbb{E}[\widehat{A}_i^{\,2}\, \rho_{i,p}^2(\theta)]$ denote the term in the bound in Eq.(\ref{eq:var-bound}), and let $p_{\text{stat}} \in [p_{\text{low}}, p_{\text{high}}]$ be any fixed parameter. Given a $y_i$ of length $n$, the dynamic schedule satisfies:

\textbf{1. Early-phase signal amplification:}  If  $y_i$ has a high-ratio token $t^{*}$ with $r_{i,t^{*}}\gg 1$, while the other tokens have constant-bounded ratios. Under the pre-saturation condition $r_{i,t^{*}}^{p_{\text{high}}} \ll n - 1$, shifting from $p_{\text{stat}}$ to $p_{\text{high}}$ exponentially amplifies its gradient weight: there exists a constant $C > 0$ such that
\begin{equation}
\frac{W_{i,t^{*}}(p_{\text{high}})}{W_{i,t^{*}}(p_{\text{stat}})} \;\ge\; C \cdot r_{i,t}^{\,p_{\text{high}} - p_{\text{stat}}}.   
\end{equation}

\textbf{2. Late-phase variance contraction:} The terminal variance bound is strictly contracted:
\begin{equation}
    V(p_{\text{low}}) \;<\; V(p_{\text{stat}}).
\end{equation}

\end{theorem}

This theorem (proof in Appendix~\ref{app:thm3}) reveals that any static parameter $p_{\text{stat}}$, the standard paradigm in current GRPO-based methods, is a compromise for long-horizon reasoning tasks: it must sacrifice either early-stage signal amplification (if $p$ is low) or late-stage variance control (if $p$ is high). Our schedule bypasses the dilemma, dynamically allocating required mechanism to each training phase.

Figure~\ref{fig:logratio} provides direct visual support for this 
choice: the per-step ratio envelopes under static $p \in \{+2, 0, -2\}$ 
illustrate how decreasing $p$ monotonically tightens the gap between 
the largest and smallest token-level ratios, and our linear schedule 
inherits the early-stage concentration of $p=+2$ while converging to the 
controlled regime of $p=-2$.




\section{Experiment}
\label{sec:experiment}
To empirically validate the effectiveness of H\"{o}lderPO, we evaluate our method against state-of-the-art policy optimisation baselines on mathematical reasoning and agentic benchmarks. Our experiments are designed to follow a clear logical progression: (1) revealing the task-specific sensitivity of the $p$ parameter on distinct benchmarks, (2) demonstrating how dynamic scheduling resolves the concentration--stability trade-off identified in Section~\ref{sec:method}, and (3) comparing our overall performance against established baselines.

\subsection{Implementation Details}
\textbf{Model.} We evaluate our framework on two task families: mathematical reasoning and agentic decision-making. For mathematical reasoning, following Dr.GRPO~\citep{liu2025understanding}, we cover a broad spectrum of base models ranging from 1.5B to 8B parameters, including the Qwen2.5-Math series (1.5B and 7B)~\citep{yang2024qwen25mathtechnicalreportmathematical}, DeepSeek-R1-Distill-Qwen-7B~\citep{guo2025deepseek}, and the Qwen3 series (4B and 8B)~\citep{yang2025qwen3}. For agentic tasks, we adopt Qwen2.5-1.5B-Instruct~\citep{qwen2025qwen25technicalreport} as the policy backbone.

\textbf{Training.} Our training pipeline follows two established protocols depending on the task. For mathematical reasoning, we adopt the recipe of Dr.GRPO~\citep{liu2025understanding}: training data consists of 8{,}523 problems from MATH~\citep{hendrycks2021measuring} (Levels 3--5), and each prompt is paired with 8 sampled rollouts capped at 3{,}000 tokens. Within each RL round, $\pi_{\theta_{\text{old}}}$ produces 1{,}024 trajectories, after which the current policy $\pi_\theta$ is refreshed 8 times using a mini-batch size of 128. For agentic tasks, we adhere to the GiGPO protocol~\citep{feng2025group} for both training and evaluation on ALFWorld. In terms of compute, all models are trained on 4$\times$H100 GPUs. We primarily compare H\"{o}lderPO against GRPO~\citep{shao2024deepseekmath}, Dr.GRPO~\citep{liu2025understanding}, and GMPO~\citep{zhao2025geometric} under matched configurations.

\textbf{Evaluation.} We report mathematical performance on five benchmarks that span a wide difficulty range. AIME24 contains 30 olympiad-level problems drawn from the 2024 American Invitational Mathematics Examination, while AMC provides 83 competition problems of intermediate difficulty. MATH500 is a 500-problem subset of MATH covering algebra, geometry, and number theory. Minerva~\citep{lewkowycz2022solving} consists of 272 graduate-level problems that demand multi-step derivations, and OlympiadBench (Oly.)~\citep{he-etal-2024-olympiadbench} collects 675 high-difficulty olympiad problems. For agentic evaluation, we use the six ALFWorld~\citep{shridhar2020alfworld} sub-task categories, namely Pick, Look, Clean, Heat, Cool, and Pick Two. Following Dr.GRPO~\citep{liu2025understanding}, we adopt Pass@1 as the primary metric for mathematical tasks and decode greedily with temperature 0.0, generating one sample per question. For ALFWorld, we report task success rate under the given standard evaluation protocol.

\subsection{Task-Specific Sensitivity of $p$}
\label{sec:task-sensitivity}
A fundamental premise of our work is that a static aggregation function cannot optimally solve all tasks. To illustrate this, we isolate the performance of H\"{o}lderPO across different static $p$ values on two benchmarks with distinct signal-density profiles: AIME24, where correct reasoning is concentrated in a small number of rare, high-magnitude tokens (sparse-signal regime), and MATH500, where supervision is more densely distributed across many tokens (dense-signal regime). 

\begin{table}[ht]
\centering
\small
\begin{tabular}{lcc}
\toprule
\textbf{Training Objectives} & \textbf{AIME24 (Sparse-Signal)} & \textbf{MATH500 (Dense-Signal)} \\
\midrule
GRPO ($p=1$)  & 40.0 & 83.4 \\
GMPO ($p \to 0$)  & 43.3 & 82.0 \\   
\midrule
H\"{o}lderPO ($p = -2$) & 36.7 & 84.6 \\
H\"{o}lderPO ($p = -1$) & 36.7 & \textbf{85.0} \\
H\"{o}lderPO ($p \rightarrow 0$) & 43.3 & 84.6 \\
H\"{o}lderPO ($p = 1$) & 40.0 & 83.2 \\
H\"{o}lderPO ($p = 2$) & 43.3 & 82.0 \\
H\"{o}lderPO ($p = 3$) & \textbf{46.7} & 81.8 \\
\bottomrule
\end{tabular}
\vspace{1mm}
\caption{Performance on benchmarks with distinct signal-density profiles. On AIME24, higher $p$ amplifies rare high-magnitude signals for complex reasoning. Conversely, on MATH500, a lower $p$ strictly tightens the gradient variance bound to ensure training stability, yielding superior performance on simpler tasks.}
\label{tab:combined_results}
\vspace{-5mm}
\end{table}

As detailed in Table~\ref{tab:combined_results} and visually summarised by the diverging performance curves in Figure~\ref{fig:placeholder}, the optimal $p$ value diverges significantly across the two regimes.

\textbf{Sparse-signal tasks favour high $p$.} On AIME24, where correct reasoning traces (i.e., pivotal reasoning steps) are exceptionally sparse, larger positive values of $p$ (e.g., $p \ge 2$) yield the highest accuracy. This empirically confirms Theorem~\ref{thm:deformation}: in the positive-concentration regime, the gradient weight distribution concentrates on tokens with the largest importance ratios (as visually depicted by the right-skewed $W(p)$ distributions at the bottom of Figure~\ref{fig:placeholder}), allowing the rare, high-quality reasoning steps to drive the update rather than being averaged out by the bulk of unremarkable tokens.

\textbf{Dense-signal tasks favour low $p$.} Conversely, on MATH500, where supervision is distributed across many tokens, lower values of $p$ (e.g., $p \le 0$) perform better. This is consistent with Theorem~\ref{thm:variance}: decreasing $p$ tightens the variance bound on the policy gradient estimator, preventing the high-variance updates that occur when relative-magnitude differences among many comparable tokens get over-weighted. This mechanism directly corresponds to the flatter, left-leaning $W_i^p$ distributions shown in Figure~\ref{fig:placeholder}, which systematically redistribute credit to underemphasised steps.

\subsection{Main Performance and Dynamic Scheduling}
The empirical observation that no single static $p$ yields optimal performance universally directly motivates our dynamic scheduling approach. We hypothesise that any reasoning task effectively functions as a sparse-signal task during the early stages of training. At this point, the model has yet to internalise the correct reasoning patterns, thus requiring a high $p$ for signal amplification. As the model masters the underlying logic, the task gradually transitions into a dense-signal regime, necessitating a low $p$ to ensure stable convergence. 

To validate this, we evaluate our dynamic annealing scheduler (employing a linear decay of $p$ from $2$ to $-2$) alongside the best static configuration and existing state-of-the-art baselines across a diverse suite of benchmarks.

\begin{figure}[htbp]
    \centering
    \includegraphics[width=1\linewidth]{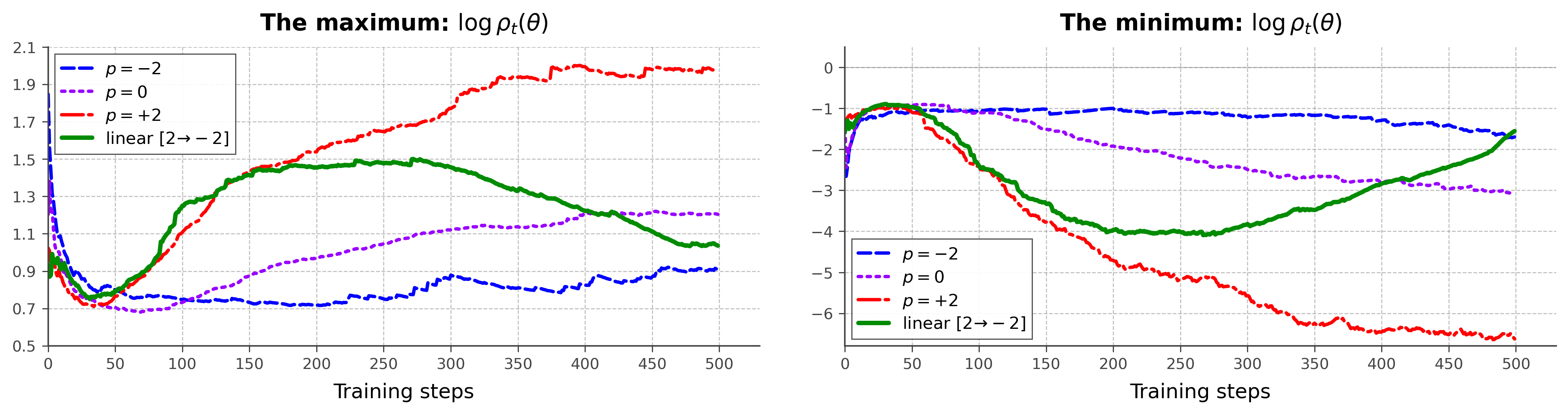}
    \vspace{-.65cm}
\caption{Token-level importance ratio $\log\rho_t(\theta)$ during training.
  \textbf{Left} and \textbf{Right} track the per-step upper and lower envelopes respectively.
  As $p$ decreases, the upper envelope drops and the lower envelope rises,
  tightening the gap monotonically. Our decaying schedule $p\!:\!2\!\to\!-2$
  (solid green) thus enables aggressive updates in the early stage and
  progressively converges to stable optimization in the later stage. Constant-$p$
  baselines ($p\!\in\!\{+2,0,-2\}$) shown as dashed/dotted/dash-dotted.}    \label{fig:logratio}
\end{figure}

Table~\ref{tab:overall_results} summarises the overall performance. While 
our best static configuration ($p \to 0$) achieves highly competitive 
average scores, it remains a single-point compromise on the 
concentration--stability trade-off. The dynamic $p$-scheduling approach 
achieves state-of-the-art results across the board: by progressively 
annealing $p$, the model leverages the early-stage signal amplification 
provided by $p = 2$, while benefiting from the strict variance contraction 
of $p = -2$ during the final convergence phase. The advantage is 
\emph{distributional} rather than pointwise: tasks whose optimal $p^{*}$ 
lies near a single endpoint may remain best served by the corresponding 
static configuration. For example, AIME24 favours static $p=3$ and MATH500 favours $p=-1$. But the schedule strictly outperforms every static setting on 
the overall task average.

\begin{table}[ht]
\centering
\resizebox{\textwidth}{!}{
\begin{tabular}{l|ccccc|c}
\toprule
\textbf{Training Objectives} & \textbf{AIME24} & \textbf{AMC} & \textbf{MATH500} & \textbf{Minerva} & \textbf{Oly.} & \textbf{Avg.} \\
\midrule
\multicolumn{7}{c}{\textit{1.5B Models}} \\
\midrule
\multicolumn{7}{l}{\textit{Base \& Instruct Models}} \\
Qwen2.5-Math-1.5B~\citep{yang2024qwen25mathtechnicalreportmathematical}          & 16.7 & 43.4 & 61.8 & 15.1 & 28.4 & 33.1 \\
Qwen2.5-Math-1.5B-Instruct~\citep{yang2024qwen25mathtechnicalreportmathematical} & 10.0 & 48.2 & 74.2 & 26.5 & 40.2 & 39.8 \\
\midrule
\multicolumn{7}{l}{\textit{RL Post-Trained Models}} \\
Oat-Zero-1.5B~\citep{liu2025understanding}              & 20.0 & 53.0 & 74.2 & 25.7 & 37.6 & 42.1 \\
GMPO-1.5B~\citep{zhao2025geometric}                  & 20.0 & 53.0 & 77.6 & 30.1 & 38.7 & 43.9 \\
\rowcolor{gray!15}
\textbf{H\"{o}lderPO-1.5B (Ours)}      & 30.0 & 48.1 & 77.0 & 27.9 & 39.1 & \textbf{44.5} \\
\midrule
\multicolumn{7}{c}{\textit{7B Models}} \\
\midrule
\multicolumn{7}{l}{\textit{Base Models}} \\
Qwen2.5-Math-7B~\citep{yang2024qwen25mathtechnicalreportmathematical}          & 16.7 & 38.6 & 50.6 & 9.9  & 16.6 & 26.5 \\
\midrule
\multicolumn{7}{l}{\textit{RL Post-Trained Models}} \\
SimpleRL-Zero-7B~\citep{zeng2025simplerl}         & 26.7 & 60.2 & 78.2 & 27.6 & 40.3 & 46.6 \\
PRIME-Zero-7B~\citep{cui2025process}             & 16.7 & 62.7 & 83.8 & 36.0 & 40.9 & 48.0 \\
OpenReasoner-Zero-7B @ 3k~\citep{hu2025open}& 13.3 & 47.0 & 79.2 & 31.6 & 44.0 & 43.0 \\
OpenReasoner-Zero-7B @ 8k~\citep{hu2025open}& 13.3 & 54.2 & 82.4 & 31.6 & 47.9 & 45.9 \\
Eurus-7B~\citep{yuan2024advancing}                 & 16.7 & 62.7 & 83.8 & 36.0 & 40.9 & 48.0 \\
GPG-7B~\citep{chu2025gpg}                   & 33.3 & 65.0 & 80.0 & 34.2 & 42.4 & 51.0 \\
Oat-Zero-7B~\citep{liu2025understanding}              & 43.3 & 62.7 & 80.0 & 30.1 & 41.0 & 51.4 \\
GRPO ($p=1$)~\citep{shao2024deepseekmath}                         & 40.0 & 59.0 & 83.4 & 32.4 & 41.3 & 51.2 \\
GMPO ($p \to 0$)~\citep{zhao2025geometric}                      & 43.3 & 61.4 & 82.0 & 33.5 & 43.6 & 52.7 \\
PMPO~\citep{zhao2026one}                            & 36.7 & 68.7 & 83.8 & 34.9 & \textbf{46.7} & 54.2 \\
\midrule
\multicolumn{7}{l}{\textit{H\"{o}lderPO (Ours)}} \\
\textbf{H\"{o}lderPO ($p = -2$)}      & 36.7 & 53.0 & 84.6 & 33.5 & 44.7 & 50.5 \\
\textbf{H\"{o}lderPO ($p = -1$)}      & 40.0 & 59.0 & \textbf{85.0} & 33.8 & 42.1 & 52.0 \\
\textbf{H\"{o}lderPO ($p \to 0$)}      & 43.3 & 57.8 & 84.6 & 31.6 & 45.5 & 52.6 \\
\textbf{H\"{o}lderPO ($p = 1$)}      & 40.0 & 57.8 & 83.2 & 30.9 & 44.9 & 51.4 \\
\textbf{H\"{o}lderPO ($p = 2$)}      & 43.3 & 55.4 & 82.0 & 31.2 & 46.5 & 51.7 \\
\textbf{H\"{o}lderPO ($p = 3$)}      & \textbf{46.7} & 61.4 & 81.8 & 32.4 & 40.9 & 52.6 \\
\rowcolor{gray!15}
\textbf{H\"{o}lderPO (Linear Des: $2 \to -2$)}      & 43.3 & \textbf{68.7} & 82.2 & \textbf{34.9} & 45.3 & \textbf{54.9} \\
\midrule
\multicolumn{7}{c}{\textit{R1-Distill-Qwen-7B}} \\
\midrule
\multicolumn{7}{l}{\textit{RL Post-Trained Models}} \\
GRPO ($p=1$)~\citep{shao2024deepseekmath}              & 43.3 & 67.5 & 89.0 & 39.7 & 56.7 & 59.3 \\
Dr.GRPO~\citep{liu2025understanding}                   & 50.0 & 74.7 & 89.6 & 37.5 & 55.7 & 61.5 \\
GMPO ($p \to 0$)~\citep{zhao2025geometric}             & 46.7 & 78.3 & 91.4 & 37.9 & 62.5 & 63.4 \\
PMPO~\citep{zhao2026one}                                & 46.7 & 79.5 & 93.4 & 39.3 & 64.2 & 64.6 \\
\rowcolor{gray!15}
\textbf{H\"{o}lderPO (Linear Des: $2 \to -2$, Ours)} & 53.3 & 79.5 & 92.6 & 42.3 & 64.1 & \textbf{66.4} \\
\bottomrule
\end{tabular}
}
\caption{Comprehensive comparison of H\"{o}lderPO against state-of-the-art baselines across different model scales and base architectures. The static rows report fixed $p$ settings, while the dynamic row employs our linear annealing scheduler, which progressively decays $p$ from an initial value of $2$ to a terminal value of $-2$ over the course of training.}
\label{tab:overall_results}
\vspace{-8mm}
\end{table}

\subsection{Selecting the Schedule Range}
\label{sec:operating-range}

Theorem~\ref{thm:dynamic} establishes the benefit of dynamic scheduling but 
leaves the endpoints $[p_{\text{low}}, p_{\text{high}}]$ open. We select 
this range based on three considerations.

\textbf{Empirical performance is concentrated in a moderate range.} The 
static sweep in Section~\ref{sec:task-sensitivity} shows that the strongest 
configurations across benchmarks fall within $p \in [-2, 2]$, with 
performance degrading smoothly outside this interval.

\textbf{The lower bound is constrained by optimisation stability.} 
Corollary~\ref{variance_optimal} refines Theorem~\ref{thm:variance}: 
under mild gradient-orthogonality, the second moment is minimised at some 
$p^{*} \le 0$ rather than $p \to -\infty$, since weight concentration grows 
exponentially and counteracts the H\"{o}lder-mean decrease. 

\textbf{The optimal range is task-dependent.} We adopt $[2, -2]$ as the 
default for mathematical reasoning, where Qwen2.5-Math's strong 
pre-training tolerates the aggressive upper bound for early-phase signal 
amplification. The endpoints are not universal: for ALFWorld 
(Section~\ref{sec:agentic}), where the base model lacks domain-specific 
pre-training, a more conservative $[1, -1]$ outperforms $[2, -2]$, 
suggesting both endpoints should be calibrated to the base model's 
reasoning maturity and the task's signal-density profile.

\subsection{Generalisation to Agentic Reasoning}
\label{sec:agentic}
To demonstrate that the advantages of H\"{o}lderPO extend beyond pure mathematical domains to broader sequential decision-making scenarios, we evaluate our framework on the ALFWorld benchmark~\citep{shridhar2020alfworld}. ALFWorld is a challenging embodied agentic environment that requires models to complete multi-step, open-ended tasks (e.g., finding, cleaning, or heating objects) based entirely on textual observations and actions. Unlike mathematical proofs, where reasoning is largely self-contained, agentic tasks suffer from compounding errors over long horizons, making stable policy optimisation crucial for success. Following established setups, we employ Qwen2.5-Instruct-1.5B as our base model for this agentic reasoning task. Table~\ref{tab:alfworld_results} presents the success rates across the six distinct sub-task categories in the ALFWorld evaluation suite.

\begin{table}[ht]
\centering
\small 
\begin{tabular}{l|cccccc|c}
\toprule
\textbf{Training Objectives} & \textbf{Pick} & \textbf{Look} & \textbf{Clean} & \textbf{Heat} & \textbf{Cool} & \textbf{Pick Two} & \textbf{Avg.} \\
\midrule
\multicolumn{8}{l}{\textit{Baselines (Base Model: Qwen2.5-Instruct-1.5B)}} \\
\midrule
GRPO ($p=1$)~\citep{shao2024deepseekmath} & 85.3 & 53.7 & 84.5 & 78.2 & 59.7 & 53.5 & 72.8 \\
GMPO ($p \to 0$)~\citep{zhao2025geometric} & 93.1 & 78.6 & 81.0 & 88.2 & 82.1 & \textbf{89.5} & 85.9 \\
GiGPO~\citep{feng2025group} & 94.4 & 67.5 & 94.8 & 94.4 & 79.8 & 76.4 & 86.7 \\
\midrule
\multicolumn{8}{l}{\textit{H\"{o}lderPO (Ours)}} \\
\midrule
\textbf{H\"{o}lderPO (Linear Dec: $2 \to -2$)} & \textbf{97.2} & 85.7 & 87.5 & 91.7 & 79.2 & 81.5 & 87.5 \\
\rowcolor{gray!15}
\textbf{H\"{o}lderPO (Linear Dec: $1 \to -1$)} & 96.9 & \textbf{100.0} & \textbf{100.0} & \textbf{100.0} & \textbf{85.7} & 84.5 & \textbf{93.8} \\
\bottomrule
\end{tabular}
\vspace{2mm}
\caption{Success rates (\%) on the ALFWorld agentic reasoning benchmark. H\"{o}lderPO demonstrates strong generalisation to open-ended, multi-step decision-making tasks.}
\vspace{-2mm}
\label{tab:alfworld_results}
\end{table}

Consistent with our findings in the mathematical domain, H\"{o}lderPO yields substantial performance gains in agentic environments. The dynamic scheduling of $p$ proves particularly well-suited for the compounding challenges of ALFWorld. During the early stages of training, a positive initial $p$ amplifies the sparse, high-magnitude signals associated with rare successful trajectories, effectively exploiting the positive-concentration regime (Theorem~\ref{thm:deformation}). In the later stages, annealing to a negative $p$ tightens the gradient variance bound (Theorem~\ref{thm:variance}), protecting the policy from being derailed by spurious environmental feedback or minor missteps. 

Notably, because our base model (Qwen2.5-Instruct-1.5B) lacks the extensive domain-specific pre-training seen in the mathematical models, it does not initially possess strong, reliable intuitions for embodied environments. Consequently, an overly aggressive initial parameter (e.g., $p=2$) risks over-amplifying early, noisy exploration. Instead, a more conservative schedule (decaying from $1$ to $-1$) proves optimal. By providing a steadier phase of signal amplification before transitioning into variance contraction, this tailored schedule achieves an exceptional average success rate of $93.8\%$, substantially outperforming all baselines. This careful calibration of the concentration--stability trade-off yields a highly robust policy for long-horizon planning.
\vspace{-3mm}
\section{Conclusion}
\label{sec:conclusion}

We introduced \textbf{H\"{o}lder Policy Optimisation (H\"{o}lderPO)}, a 
generalised framework that resolves the concentration--stability trade-off 
inherent in static policy optimisation methods like GRPO. By unifying 
importance-ratio aggregation through the H\"{o}lder mean, the parameter $p$ 
serves as a continuous dial: larger $p$ amplifies sparse high-magnitude 
learning signals, while smaller $p$ tightens the gradient variance bound. 
Built on this principle, our dynamic $p$-annealing scheduler achieves 
state-of-the-art performance across mathematical and agentic benchmarks, 
securing $54.9\%$ on five 
mathematical reasoning benchmarks and $93.8\%$ on ALFWorld.

\paragraph{Limitations.}
Two limitations stand out. First, the schedule introduces hyperparameters 
($p_{\text{high}}$, $p_{\text{low}}$, decay shape) that require empirical 
tuning per task; while linear decay performed best in our setup, we 
provide no theoretical characterisation of the optimal shape. Second, the 
positive-concentration regime amplifies tokens with high importance ratios, 
making H\"{o}lderPO more susceptible to reward hacking when the verifier 
provides false-positive signals.

\paragraph{Future Work.}
A primary direction is an \textit{adaptive} scheduler that adjusts $p$ from 
real-time metrics (e.g., batch-level gradient variance or token-ratio 
dispersion), removing the need for manual tuning.

\bibliographystyle{plainnat}

\bibliography{reference}


\appendix

%
%

\section*{Appendix Outline}
\begin{itemize}[leftmargin=*]
    \item Appendix~\ref{app:extended_related_work}: Extended related work.
    \item Appendix~\ref{app:dynamics}: Training dynamics (entropy and gradient-norm) under different $p$.
    \item Appendix~\ref{app:pseudocode}: Pseudocode for the H\"{o}lderPO loss in log-space.
    \item Appendix~\ref{app:token_clip}: H\"{o}lderPO results under token-level clipping.
    \item Appendix~\ref{sec:schedule-ablation}: Schedule-shape ablation across linear, square, cube, and sinusoidal interpolations.
    \item Appendix~\ref{app:qwen3}: Generalisation to Qwen3-4B-Base and Qwen3-8B-Base.
    \item Appendix~\ref{app:gradient}: Formulas and gradient derivations under three clipping regimes.
    \item Appendix~\ref{app:thm1}: Proofs of Theorem~\ref{thm:amplification} and Theorem~\ref{thm:deformation} (distribution deformation).
    \item Appendix~\ref{app:thm2}: Proof of Theorem~\ref{thm:variance} and Corollary~\ref{variance_optimal} (variance behaviours).
    \item Appendix~\ref{app:thm3}: Proof of Theorem~\ref{thm:dynamic} (advantage of dynamic scheduling).
    \item Appendix~\ref{broader impacts}: Broader Impacts
\end{itemize}
\newpage

\section{Extended Related Work}
\label{app:extended_related_work}

We expand on the broader GRPO ecosystem briefly mentioned in Section~\ref{sec:related_work}. The variants below address aspects of the RL pipeline orthogonal to token-level aggregation.

\textbf{Surrogate Loss and Critic-Free Variants.} 
GPG~\citep{chu2025gpg} simplifies the GRPO objective by removing surrogate losses entirely, while DAPO~\citep{yu2025dapo} introduces dynamic sampling and decoupled clipping bounds. Dr.GRPO~\citep{liu2025understanding} mitigates length bias by removing the per-sequence length normalisation, $\lambda$-GRPO~\citep{wang2025lambda} learns the length preference via a trainable parameter. These methods modify the loss normalisation rather than the aggregation function and are complementary to our framework.

\textbf{Advantage Estimation and Reward Shaping.}
AAPO~\citep{xiong2025aapo} introduces advantage momentum to mitigate zero-gradient situations; BNPO~\citep{xiao2025bnpo} adaptively normalises rewards via a Beta distribution; OPO~\citep{hao2025policy} provides a variance-minimising baseline; and Seed-GRPO~\citep{chen2025seed} scales updates by question difficulty. These contributions modify the advantage signal rather than how token-level ratios are aggregated.

\textbf{Value-Model-Based Variants.} 
To circumvent GRPO's variance issues, some approaches revert to PPO with pre-trained value models, including VC-PPO~\citep{yuan2025s} and T-PPO~\citep{ouyang2025token}. While effective, the external value model introduces confounding factors and computational overhead that the critic-free GRPO family, including ours, deliberately avoids.

\textbf{Data-Centric and Curriculum-Based Approaches.}
Open-Reasoner-Zero~\citep{hu2025open}, PRIME~\citep{cui2025process}, and SimpleRL-Zero~\citep{zeng2025simplerl} democratise scalable RL training through curated training data, curriculum learning, and clean base models. These contributions are at the data and pipeline level, complementary to algorithmic refinements such as ours.

\section{Training Dynamics: Entropy and Gradient-Norm}
\label{app:dynamics} 
\vspace{-5mm}
\begin{figure}[ht]
    \centering
    \includegraphics[width=1\linewidth]{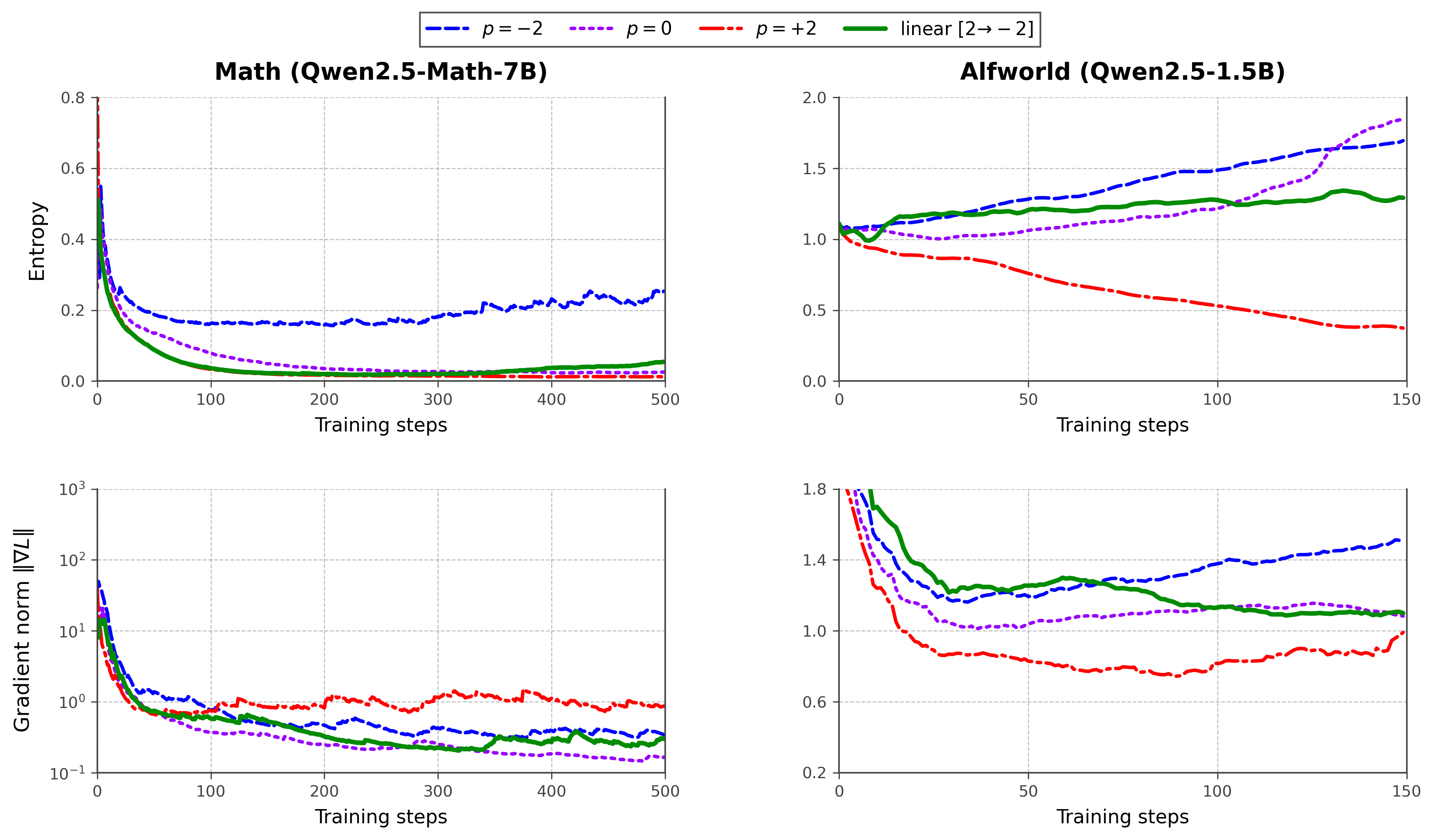}
    \vspace{-.65cm}
    \caption{\textbf{Entropy and gradient-norm dynamics under different H\"{o}lder exponents $p$.}
  \textbf{Columns:} Math (Qwen2.5-Math-7B on MATH-12k) and Alfworld
  (Qwen2.5-1.5B). \textbf{Rows:} per-step policy entropy and gradient
  norm $\|\nabla\mathcal{L}\|$ (log scale on Math, linear on Alfworld).
  Constant-$p$ baselines ($p\!\in\!\{+2,0,-2\}$, dashed/dotted/dash-dotted) are
  compared with our linearly-decaying schedule $p\!:\!2\!\to\!-2$ (solid green).
  Positive $p$ concentrates mass on high-likelihood tokens and pushes entropy
  down; negative $p$ disperses mass and pushes it up. The schedule inherits
  both regimes in sequence and keeps the gradient norm in a tighter band than
  any constant choice.}   
\label{fig:traindynamic}
\end{figure}

\section{Pseudocode}
\label{app:pseudocode} 

H\"{o}lderPO is a single-line modification of the GRPO loss. To preserve numerical stability for large $|p|$, all power operations are evaluated in log-space via the log-sum-exp identity. Algorithm~\ref{alg:holder_po} summarises the full computation. The aggregation operator is applicable at any granularity: in our experiments we use a single sequence-level $\rho_{i,p}$, but token-level or block-level aggregation can be substituted without changing the algorithm or theory.

\begin{algorithm}[ht]
\caption{H\"{o}lderPO Loss Computation}
\label{alg:holder_po}
\begin{algorithmic}[1]
\STATE \textbf{Require:} current policy $\pi_\theta$, reference policy $\pi_{\theta_{\text{old}}}$
\STATE \textbf{Input:} sequence $y$ of length $T$, valid-token mask $M$, advantage $\widehat{A}$, parameter $p$, clip range $\epsilon$
\STATE \textcolor{gray}{\texttt{// Step 1: log-ratio computation}}
\STATE $\Delta_t \gets \log \pi_\theta(y_t \mid x, y_{<t}) - \log \pi_{\theta_{\text{old}}}(y_t \mid x, y_{<t})$
\STATE $|y| \gets \sum_{t=1}^T M_t$
\STATE \textcolor{gray}{\texttt{// Step 2: H\"{o}lder-mean aggregation in log-space}}
\IF{$|p| < 10^{-6}$}
    \STATE $\rho \gets \exp\!\left( \tfrac{1}{|y|} \sum_t M_t\, \Delta_t \right)$ \hfill \textcolor{gray}{\texttt{// limit $p \to 0$ (geometric mean)}}
\ELSE
    \STATE $\rho \gets \left( \tfrac{1}{|y|} \sum_t M_t\, \exp(p\,\Delta_t) \right)^{\!1/p}$
\ENDIF
\STATE \textcolor{gray}{\texttt{// Step 3: PPO-style clipping and loss}}
\STATE $\rho_{\text{clip}} \gets \mathrm{clip}(\rho,\, 1-\epsilon,\, 1+\epsilon)$
\STATE $\mathcal{L}_{\text{unclip}} \gets -\widehat{A}\cdot\rho$
\STATE $\mathcal{L}_{\text{clip}}   \gets -\widehat{A}\cdot\rho_{\text{clip}}$
\STATE $\mathcal{L}_{\text{HPO}}     \gets \max\!\big(\mathcal{L}_{\text{unclip}},\; \mathcal{L}_{\text{clip}}\big)$ \hfill \textcolor{gray}{\texttt{// minimised via SGD}}
\STATE \textbf{return} $\mathcal{L}_{\text{HPO}}$
\end{algorithmic}
\end{algorithm}

\section{Token-Level Clipping of H\"{o}lderPO}
\label{app:token_clip} 
\begin{table}[h]
\centering
\resizebox{\textwidth}{!}{
\begin{tabular}{l|ccccc|c}
\toprule
\textbf{Training Objectives} & \textbf{AIME24} & \textbf{AMC} & \textbf{MATH500} & \textbf{Minerva} & \textbf{Oly.} & \textbf{Avg.} \\
\midrule
\multicolumn{7}{l}{\textit{Token-Level Clip H\"{o}lderPO}} \\
\midrule
\textbf{H\"{o}lderPO ($p = -2$)}      & 36.7 & 61.4 & 81.6 & 35.7 & 43.6 & 51.8 \\
\textbf{H\"{o}lderPO ($p = -1$)}      & 40.0 & 62.7 & 81.6 & 33.5 & 44.5 & 52.5 \\
\textbf{H\"{o}lderPO ($p \to 0$)}     & 43.3 & 61.3 & 81.6 & 34.6 & 42.7 & 52.7 \\
\textbf{H\"{o}lderPO ($p = 1$)}       & 43.3 & 63.9 & 80.8 & 31.6 & 42.8 & 52.5 \\
\textbf{H\"{o}lderPO ($p = 2$)}       & 43.3 & 60.2 & 81.2 & 33.5 & 44.9 & 52.6 \\
\bottomrule
\end{tabular}
}
\caption{H\"{o}lderPO under \textit{token-level} clipping (Eq.~\ref{tokenclippob}) on Qwen2.5-Math-7B across five mathematical benchmarks. Compared with the sequence-level clipping setting reported in Table~\ref{tab:overall_results}, the token-level variant produces a noticeably narrower performance spread across $p$, consistent with our discussion in Section~\ref{sec:variance}: token-level clipping breaks the algebraic structure underlying the variance bound's monotonicity in $p$, weakening the controlled trade-off that motivates dynamic scheduling.}
\label{tab:token_overall_results}
\vspace{-2mm}
\end{table}

\newpage
\section{Schedule-Shape Ablation}
\label{sec:schedule-ablation}
\begin{table}[h]
\centering

\resizebox{\textwidth}{!}{
\begin{tabular}{l|ccccc|c}
\toprule
\textbf{Training Objectives} & \textbf{AIME24} & \textbf{AMC} & \textbf{MATH500} & \textbf{Minerva} & \textbf{Oly.} & \textbf{Avg.} \\
\midrule
\multicolumn{7}{l}{\textit{Dynamic H\"{o}lderPO Variants}} \\
\midrule
\textbf{H\"{o}lderPO (Square Asc: $-2 \to 2$)} & 33.3 & 62.7 & 82.6 & 33.8 & 44.4 & 51.4 \\
\textbf{H\"{o}lderPO (Cube Asc: $-2 \to 2$)}   & 36.7 & 62.7 & 82.8 & 32.0 & 43.1 & 51.4 \\
\textbf{H\"{o}lderPO (Sin Asc: $-2 \to 2$)}    & 43.3 & 59.0 & 81.6 & 35.7 & 45.8 & \textbf{53.1} \\
\textbf{H\"{o}lderPO (Linear Asc: $-2 \to 2$)} & 40.0 & \textbf{66.3} & 81.2 & 32.7 & 42.7 & 52.6 \\
\textbf{H\"{o}lderPO (Square Des: $2 \to -2$)} & 36.7 & 61.4 & 81.6 & 33.1 & 44.0 & 51.3 \\
\textbf{H\"{o}lderPO (Cube Des: $2 \to -2$)}   & 36.7 & 60.2 & 80.8 & 31.6 & \textbf{46.1} & 51.1 \\
\textbf{H\"{o}lderPO (Sin Des: $2 \to -2$)}    & 36.7 & 61.4 & 81.6 & 35.3 & 46.8 & 52.4 \\
\bottomrule
\end{tabular}
}
\caption{Comparison of alternative annealing shapes for the dynamic schedule on Qwen2.5-Math-7B. We sweep four monotonic interpolation families (linear, square, cube, sinusoidal) in both ascending ($-2 \to 2$) and descending ($2 \to -2$) directions, holding the endpoints fixed at $\{-2, +2\}$. Among the seven variants listed here, none surpasses the descending linear schedule of $54.9\%$ reported in Table~\ref{tab:overall_results}, supporting our choice of linear decay as the default.}
\label{tab:ablation_overall_results}
\vspace{-2mm}
\end{table}

\section{Generalisation to Qwen3 Base Models}
\label{app:qwen3}  
To verify that H\"{o}lderPO transfers beyond the Qwen2.5-Math-7B setting, we additionally evaluate on the Qwen3-Base series and compare against three strong token-aggregation baselines (GRPO, GSPO, DAPO) under matched configurations.

\begin{table}[ht]
\centering
\resizebox{\textwidth}{!}{
\begin{tabular}{l|ccccc|c}
\toprule
\textbf{Training Objectives} & \textbf{MATH500} & \textbf{AIME25$^\dagger$} & \textbf{AMC23} & \textbf{Minerva} & \textbf{Oly.} & \textbf{Avg.} \\
\midrule
\multicolumn{7}{l}{\textit{Base Model}} \\
\midrule
Qwen3-4B-Base~\citep{yang2025qwen3} & 58.2 & 7.4  & 45.0 & 14.0 & 28.6 & 30.6 \\
\midrule
\multicolumn{7}{l}{\textit{RL Post-Trained Models}} \\
\midrule
GRPO~\citep{shao2024deepseekmath} & 79.3 & 18.5 & 60.0 & 21.0 & 40.5 & 43.9 \\
GSPO~\citep{zheng2025group}       & 78.5 & 18.5 & 62.5 & 23.2 & 39.8 & 44.5 \\
DAPO~\citep{yu2025dapo}           & 81.3 & 22.2 & \textbf{65.0} & 21.7 & \textbf{41.8} & 46.4 \\
\rowcolor{gray!15}
\textbf{H\"{o}lderPO (Linear Des: $2 \to -2$, Ours)} & \textbf{88.0} & \textbf{30.0} & 60.2 & \textbf{39.0} & 40.6 & \textbf{50.9} \\
\bottomrule
\end{tabular}
}
\caption{Results on \textbf{Qwen3-4B-Base}. We report Pass@1 accuracy (\%) for MATH500, AMC23, Minerva, and Olympiad, and Pass@8 accuracy (\%) for AIME25 ($^\dagger$). H\"{o}lderPO with our default linear annealing schedule ($p\colon 2 \to -2$) achieves $50.9\%$ average accuracy, a $4.5$-point absolute gain over the strongest aggregation baseline DAPO ($46.4\%$) and $7.0$-point gain over GRPO ($43.9\%$).}
\label{tab:qwen3_4b_results}
\vspace{-2mm}
\end{table}

\begin{table}[ht]
\centering
\resizebox{\textwidth}{!}{
\begin{tabular}{l|ccccc|c}
\toprule
\textbf{Training Objectives} & \textbf{MATH500} & \textbf{AIME25$^\dagger$} & \textbf{AMC23} & \textbf{Minerva} & \textbf{Oly.} & \textbf{Avg.} \\
\midrule
\multicolumn{7}{l}{\textit{Base Model}} \\
\midrule
Qwen3-8B-Base~\citep{yang2025qwen3} & 65.0 & 11.1 & 45.0 & 17.3 & 31.1 & 33.9 \\
\midrule
\multicolumn{7}{l}{\textit{RL Post-Trained Models}} \\
\midrule
GRPO~\citep{shao2024deepseekmath} & 80.1 & 22.2 & 67.5 & 27.6 & 42.4 & 48.0 \\
GSPO~\citep{zheng2025group}       & 81.7 & 22.2 & 67.5 & 26.8 & 45.5 & 48.7 \\
DAPO~\citep{yu2025dapo}           & 85.3 & 25.9 & 75.0 & 27.9 & 48.7 & 52.6 \\
\rowcolor{gray!15}
\textbf{H\"{o}lderPO (Linear Des: $2 \to -2$, Ours)} & \textbf{88.4} & \textbf{33.3} & \textbf{75.1} & \textbf{37.5} & \textbf{50.3} & \textbf{56.9} \\
\bottomrule
\end{tabular}
}
\caption{Results on \textbf{Qwen3-8B-Base}. Same evaluation protocol as Table~\ref{tab:qwen3_4b_results}. The advantage of H\"{o}lderPO grows with model scale: at 8B, our method reaches $56.9\%$ average accuracy, a $4.3$-point gain over DAPO ($52.6\%$) and $8.9$-point gain over GRPO ($48.0\%$). Notably, the relative improvement on Minerva ($+9.6$ over DAPO) and Olympiad ($+1.6$) confirms that the gains generalise across both standard and competition-level mathematical reasoning.}
\label{tab:qwen3_8b_results}
\vspace{-2mm}
\end{table}

\section{Formulas and Derivation}
\label{app:gradient}
In this section, we derive the formulas involved in our theory. These formulas can be divided into three parts according to different clipping mechanisms: the original unclipped, token-level clipping, and sequence-level clipping. Finally, we discuss the definition of H\"{o}lder $p$-norm when $p=0$.

We define some notation used throughout this chapter. Let $\mathcal{D}$ denote the dataset of input prompts, from which a query $q$ (or context $x$) is sampled. For each prompt, we sample a group of $G$ responses, denoted as $\{y_i\}_{i=1}^G$, from the reference or old policy $\pi_{\theta_{old}}$. For the $i$-th response $y_i$, let $|y_i|$ represent its total token length, where $y_{i,t}$ is the $t$-th token and $y_{i,<t}$ denotes the prefix. The current policy parameterised by $\theta$ is denoted as $\pi_\theta$.  Finally, $\widehat{A}_i$ represents the estimated advantage for the $i$-th response, defined as 
\[
\widehat{A}_i=\frac{r\left(x, y_i\right)-\operatorname{mean}\left(\left\{r\left(x, y_i\right)\right\}_{i=1}^G\right)}{\operatorname{std}\left(\left\{r\left(x, y_i\right)\right\}_{i=1}^G\right)},
\]
where $r(x, y_i)$ denotes the absolute reward score assigned to the $i$-th generated response $y_i$ conditioned on the input $x$. Here  $\operatorname{mean}(\cdot)$ and $\operatorname{std}(\cdot)$ represent the arithmetic mean and standard deviation. For simplicity, in this section we omit the KL regularisation term from all PPO-style objective function formulas.

\subsection{No Clipping Formulas}\label{uncliformulas}
As we know, the simplest unclipped GRPO objective function formula is 
\[
\mathbb{E}_{q \sim \mathcal{D}, \{y_i\}_{i=1}^G \sim \pi_{\theta_{old}}} \left[ \frac{1}{G} \sum_{i=1}^G \left( \frac{1}{|y_i|} \sum_{t=1}^{|y_i|} \frac{\pi_\theta(y_{i,t} | x, y_{i,<t})}{\pi_{\theta_{old}}(y_{i,t} | x, y_{i,<t})} \right) \widehat{A}_i \right],
\]
which can be regarded as a special case of the objective function given by \cite{schulman2015trust} with the surrogate term $\frac{1}{|y_i|} \sum_{t=1}^{|y_i|} \frac{\pi_\theta(y_{i,t} | x, y_{i,<t})}{\pi_{\theta_{old}}(y_{i,t} | x, y_{i,<t})}$. Obviously it is the arithmetic mean value of importance sampling ratios $r_{i,t}=\frac{\pi_\theta(y_{i,t} | x, y_{i,<t})}{\pi_{\theta_{old}}(y_{i,t} | x, y_{i,<t})}$ with respect to tokens in the sequence $y_i$. By H\"{o}lder-$p$ norm, we extend the arithmetic mean value of ratios to $\rho_{i,p}(\theta)$, which is defined as 
\begin{equation}\tag{\ref{eq:rho-def}}
\rho_{i,p}(\theta) \;=\; \left( \frac{1}{|y_i|} \sum_{t=1}^{|y_i|} r_{i,t}(\theta)^p \right)^{\!1/p}, \qquad p \in \mathbb{R}\setminus\{0\}. 
\end{equation}
Later we discuss the case for $p=0$ in \ref{F4}. The unclipped objective function of H\"{o}lder-MPO is
\begin{equation}\label{unclip}
\mathcal{J}_{\mathrm{H}}(\theta)=\mathbb{E}_{x \sim \mathcal{D},\left\{y_i\right\}_{i=1}^G \sim \pi_{\theta_{\mathrm{old}}}(\cdot \mid x)}\left[\frac{1}{G} \sum_{i=1}^{\mathrm{G}} \rho_{i, p}(\theta) \widehat{A}_i\right].
\end{equation}
To calculate the policy gradient of this objective function, we first prove
\begin{equation}\label{nablerho}
\nabla_\theta \rho_{i, p}(\theta)=\frac{\rho_{i, p}(\theta)^{1-p}}{\left|y_i\right|} \sum_{t=1}^{\left|y_i\right|} r_{i, t}(\theta)^p \nabla_\theta \log \pi_\theta\left(y_{i, t} \mid x, y_{i,<t}\right).
\end{equation}
Starting from the definition of the H\"{o}lder mean in~\eqref{eq:rho-def},
\begin{equation*}
    \log \rho_{i,p}(\theta) \;=\; \frac{1}{p}\log\!\left( \frac{1}{|y_i|}\sum_{t=1}^{|y_i|} r_{i,t}(\theta)^p \right).
\end{equation*}
Differentiating both sides with respect to $\theta$:
\begin{equation*}
    \frac{\nabla_\theta \rho_{i,p}(\theta)}{\rho_{i,p}(\theta)}
    \;=\; \frac{1}{p}\, \cdot\, \frac{\nabla_\theta \sum_t r_{i,t}^p(\theta)}{\sum_t r_{i,t}^p(\theta)}.
\end{equation*}
Since $\pi_{{{old}}}$ does not depend on $\theta$, the chain rule gives
\begin{equation*}
    \nabla_\theta r_{i,t}(\theta)
    \;=\; \frac{\nabla_\theta \pi_\theta(y_{i,t}\!\mid\!x, y_{i,<t})}{\pi_{\theta_{\text{old}}}(y_{i,t}\!\mid\!x, y_{i,<t})}
    \;=\; r_{i,t}(\theta)\cdot \nabla_\theta \log \pi_\theta(y_{i,t}\!\mid\!x, y_{i,<t}),
\end{equation*}
and consequently $\nabla_\theta r_{i,t}^p(\theta) = p\, r_{i,t}^p(\theta) \, \nabla_\theta \log \pi_\theta(y_{i,t}\!\mid\!\cdot)$.
Substituting and cancelling the factor of $p$:
\begin{equation*}
    \frac{\nabla_\theta \rho_{i,p}(\theta)}{\rho_{i,p}(\theta)}
    \;=\; \frac{\sum_t r_{i,t}^p(\theta)\, \nabla_\theta \log \pi_\theta(y_{i,t}\!\mid\!\cdot)}{\sum_t r_{i,t}^p(\theta)}
    \;=\; \sum_t W_{i,t}(p)\, \nabla_\theta \log \pi_\theta(y_{i,t}\!\mid\!\cdot),
\end{equation*}
where $W_{i,t}(p)$ is defined in~\eqref{eq:rho-grad}. Multiplying through by $\rho_{i,p}(\theta)$ recovers Eq.~\eqref{eq:rho-grad}.

Then the policy gradient of the unclipped H\"{o}lder-MPO is 
\begin{equation}\label{pgunclip}
\nabla_\theta \mathcal{J}_{\mathrm{H}}(\theta)= \mathbb{E}_{x \sim \mathcal{D},\left\{y_i\right\}_{i=1}^G \sim \pi_{\theta_{\mathrm{old}}}(\cdot \mid x)} \left[\frac{1}{G} \sum_{i=1}^{\mathrm{G}} \nabla_\theta \rho_{i, p}(\theta) \widehat{A}_i\right],
\end{equation}
whose unbiased mini-batch estimator is denoted by $\widehat{\nabla}_\theta \mathcal{J}_{\mathrm{H}}(\theta)$ 
\begin{equation*}
\widehat{\nabla}_\theta \mathcal{J}_{\mathrm{H}}(\theta) = \frac{1}{B} \sum_{b=1}^{B} \left[ \frac{1}{G} \sum_{i=1}^{\mathrm{G}} \nabla_\theta \rho_{i, p}(\theta) \widehat{A}_i \right].     
\end{equation*}
where $B$ denotes the batch size. By Eq.~\eqref{nablerho}, we know
\begin{equation}\label{unclibe}
\widehat{\nabla}_\theta \mathcal{J}_{\mathrm{H}}(\theta)= \frac{1}{B} \sum_{b=1}^{B} \frac{1}{G} \sum_{i=1}^{\mathrm{G}}  \left(\widehat{A}_i \frac{\rho_{i, p}(\theta)^{1-p}}{|y_i|}\right) \sum_{t=1}^{|y_i|} r_{i, t}(\theta)^p \nabla_\theta \log \pi_\theta(y_{i, t} \mid x, y_{i,<t}).    
\end{equation}

\subsection{Token-Level Clipping Formulas}\label{tlclipping}
There are many PPO extensions adopting token-level clipping mechanisms to ensure training stability and prevent policy collapse. For instance, Group Relative Policy Optimisation (GRPO), Geometric-Mean Policy Optimisation (GMPO) \citep{zhao2025geometric} and Dynamic Sampling Policy Optimisation (DAPO) \citep{yu2025dapo}. With token-level clipping, the objective function (\ref{unclip}) of our H\"{o}lder-MPO becomes
\begin{align}\label{tokenclippob}
\mathcal{J}_{\mathrm{H}_t}(\theta) &=\mathbb{E}_{x \sim \mathcal{D},\left\{y_i\right\}_{i=1}^G \sim \pi_{\theta_{\mathrm{old}}}(\cdot \mid x)}\left[\frac{1}{G}\left(\sum_{\widehat{A}_i>0} C_{i,p} \widehat{A}_i+\sum_{\widehat{A}_i<0} D_{i,p} \widehat{A}_i \right)\right], \\
C_{i,p} & =\left(\frac{1}{\left|y_i\right|} \sum_{t=1}^{\left|y_i\right|} \min \left(r_{i, t}(\theta), \operatorname{clip}\left(r_{i, t}(\theta), 1-\varepsilon, 1+\varepsilon\right)\right)^p\right)^{1 / p}, \nonumber \\
D_{i,p} & =\left(\frac{1}{\left|y_i\right|} \sum_{t=1}^{\left|y_i\right|} \max \left(r_{i, t}(\theta), \operatorname{clip}\left(r_{i, t}(\theta), 1-\varepsilon, 1+\varepsilon\right)\right)^p\right)^{1 / p}, \nonumber   
\end{align}
where the clipping function is defined by 
\begin{equation}
\operatorname{clip}(x, 1-\epsilon, 1+\epsilon):= 
\begin{cases} 
1-\epsilon, & \text{if } x < 1-\epsilon \\
x, & \text{if } 1-\epsilon \le x \le 1+\epsilon \\
1+\epsilon, & \text{if } x > 1+\epsilon.
\end{cases}
\end{equation}
To deduce this formula, we firstly recall the token-level clipping GRPO objective function in \citep{shao2024deepseekmath}
\begin{equation*}
\mathcal{J}_{\mathrm{GRPO}}(\theta)=\mathbb{E}_{x \sim \mathcal{D},\left\{y_i\right\}_{i=1}^G \sim \pi_{\theta_{\mathrm{old}}}(\cdot \mid x)} \left[ \frac{1}{G} \sum_{i=1}^G  \frac{1}{|y_i|} \sum_{t=1}^{|y_i|} \min( r_{i,t}\widehat{A}_{i,t},\clip(r_{i,t}, 1-\epsilon, 1+\epsilon)\widehat{A}_{i,t})\right],  
\end{equation*}
where $\widehat{A}_{i,t}=\widehat{A}_i$ is the estimator of sequence-level reward. According to the sign of $\widehat{A}_{i}$, the content inside the expectation of GRPO objective function should equal to  
\[
\frac{1}{G} \left( \left(\sum_{\widehat{A}_{i}>0}+ \sum_{\widehat{A}_{i}<0} \right)  \frac{1}{|y_i|} \sum_{t=1}^{|y_i|} \min( r_{i,t}\widehat{A}_{i},\clip(r_{i,t}, 1-\epsilon, 1+\epsilon)\widehat{A}_{i}) \right).
\]
For $\widehat{A}_{i}>0$, it is obvious that 
\[
\min( r_{i,t}\widehat{A}_{i},\clip(r_{i,t}, 1-\epsilon, 1+\epsilon)\widehat{A}_{i})=\min( r_{i,t},\clip(r_{i,t}, 1-\epsilon, 1+\epsilon))\widehat{A}_{i}.
\]
For $\widehat{A}_{i}<0$, it is obvious that
\[
\min( r_{i,t}\widehat{A}_{i},\clip(r_{i,t}, 1-\epsilon, 1+\epsilon)\widehat{A}_{i})=\max( r_{i,t},\clip(r_{i,t}, 1-\epsilon, 1+\epsilon))\widehat{A}_{i}.
\]
Therefore, the positive and negative part of the content inside the expectation of GRPO objective function should be expressed as
\[
\frac{1}{G} \sum_{\widehat{A}_{i}>0}  \left(\frac{1}{|y_i|} \sum_{t=1}^{|y_i|} \min( r_{i,t},\clip(r_{i,t}, 1-\epsilon, 1+\epsilon))\right) \widehat{A}_{i}, 
\]
\[
\quad \frac{1}{G} \sum_{\widehat{A}_{i}<0}  \left(\frac{1}{|y_i|} \sum_{t=1}^{|y_i|} \max( r_{i,t},\clip(r_{i,t}, 1-\epsilon, 1+\epsilon))\right) \widehat{A}_{i},
\]
which are special cases of $C_{i,p}$ and $D_{i,p}$ when $p=1$. Later in \ref{F4} we will show the objective function of GMPO is a special case of (\ref{tokenclippob}) when $p=0$.

Next we deduce the policy gradient formula of token-level clipping objective function. It is obvious that
\[
\nabla_\theta\mathcal{J}_{\mathrm{H}_t}(\theta) =\mathbb{E}_{x \sim \mathcal{D},\left\{y_i\right\}_{i=1}^G \sim \pi_{\theta_{\mathrm{old}}}(\cdot \mid x)}\left[\frac{1}{G}\left(\sum_{\widehat{A}_i>0} \nabla_\theta C_{i,p} \widehat{A}_i+\sum_{\widehat{A}_i<0} \nabla_\theta D_{i,p} \widehat{A}_i \right)\right].
\]
The derivatives of $C_{i,p}$ and $D_{i,p}$ depend on the value taken by the clipping
function. When $r_{i,t}\le 1+\epsilon$, the smaller one of $r_{i,t}$ and $\clip(r_{i,t}, 1-\epsilon, 1+\epsilon)$ is $r_{i,t}$, whereas the smaller value becomes $1+\epsilon$ when $r_{i,t}> 1+\epsilon$. In the former case, the contribution of token $t$ to $\nabla_\theta C_{i,p}$ is 
\[
\frac{C_{i,p}(\theta)^{1-p}}{\left|y_i\right|}  r_{i, t}(\theta)^p \cdot \nabla_\theta \log \pi_\theta\left(y_{i,t} \mid x, y_{i,<t}\right).
\]
In the latter case, this token's partial derivative contribution to $\nabla_\theta C_{i,p}$ is zero. Similarly, when $r_{i,t} \ge 1-\epsilon$, the larger one of $r_{i,t}$ and $\clip(r_{i,t}, 1-\epsilon, 1+\epsilon)$ is $r_{i,t}$, whereas the larger value becomes $1-\epsilon$ when $r_{i,t}< 1-\epsilon$. In the former case, the contribution of token $t$ to $\nabla_\theta D_{i,p}$ is 
\[
\frac{D_{i,p}(\theta)^{1-p}}{\left|y_i\right|} r_{i, t}(\theta)^p \cdot \nabla_\theta \log \pi_\theta\left(y_{i, t} \mid x, y_{i,<t}\right).
\]
In the latter case, this token's partial derivative contribution is zero. We summarise all the cases in the following formula
\begin{align}
\nabla_\theta \mathcal{J}_{\mathrm{H}_t}(\theta)&=\mathbb{E}_{x,\left\{y_i\right\}}\left[\frac{1}{G} \sum_{i=1}^G \widehat{A}_i \cdot \frac{H_{i,p}(\theta)^{1-p}}{\left|y_i\right|} \sum_{t=1}^{\left|y_i\right|} \mathbb{I}_{i, t}(\theta) \cdot r_{i, t}(\theta)^p \cdot \nabla_\theta \log \pi_\theta\left(y_{i,t} \mid x, y_{i,<t}\right)\right], \label{tlclippg} \\
H_{i,p}(\theta)&= \begin{cases}C_{i,p}, & \text { if } \widehat{A}_i \geq 0 \\ D_{i,p}, & \text { if } \widehat{A}_i<0, \end{cases} \quad \mathbb{I}_{i, t}(\theta)= \begin{cases}0, & \text { if } \widehat{A}_i>0 \text { and } r_{i, t}(\theta)>1+\epsilon, \text { or, if } \widehat{A}_i<0 \text { and } r_{i, t}(\theta)<1-\epsilon \\ 1, & \text { otherwise. } \nonumber\end{cases}
\end{align}
The unbiased mini-batch estimator is 
\begin{equation}\label{tles}
\widehat{\nabla}_\theta \mathcal{J}_{\mathrm{H}_t}(\theta)= \frac{1}{B} \sum_{b=1}^{B} \frac{1}{G} \sum_{i=1}^G \widehat{A}_i \cdot \frac{H_{i,p}(\theta)^{1-p}}{\left|y_i\right|} \sum_{t=1}^{\left|y_i\right|} \mathbb{I}_{i, t}(\theta) \cdot r_{i, t}(\theta)^p \cdot \nabla_\theta \log \pi_\theta\left(y_{i,t} \mid x, y_{i,<t}\right).
\end{equation}
\subsection{Sequence-Level Clipping Formulas}\label{slclipping}
Notably, alongside the widespread adoption of token-level clipping, several recent studies have shifted towards sequence-level clipping strategies, such as Group Sequence Policy Optimisation (GSPO) \citep{zheng2025group}, whose objective function is 
\[
\mathcal{J}_{\mathrm{GSPO}}(\theta)=\mathbb{E}_{x \sim \mathcal{D},\left\{y_i\right\}_{i=1}^G \sim \pi_{\theta_{\mathrm{old}}}(\cdot \mid x)}\left[\frac{1}{G} \sum_{i=1}^G \min \left(s_i(\theta) \widehat{A}_i, \operatorname{clip}\left(s_i(\theta), 1-\varepsilon, 1+\varepsilon\right) \widehat{A}_i\right)\right],
\]
where $s_i(\theta)=\left(\prod_{t=1}^{|y_i|} r_{i,t}\right)^{\frac{1}{|y_i|}}=\exp \left(\frac{1}{\left|y_i\right|} \sum_{t=1}^{\left|y_i\right|} \log \frac{\pi_\theta\left(y_{i, t} \mid x, y_{i,<t}\right)}{\pi_{\theta_{\text {old }}}\left(y_{i, t} \mid x, y_{i,<t}\right)}\right)$ is the geometric mean of ratio of each token. Actually it is a special case of our H\"{o}lder-MPO objective function with sequence-level clipping 
\begin{equation}
\mathcal{J}_{\mathrm{H}_s}(\theta)=\mathbb{E}_{x \sim \mathcal{D},\left\{y_i\right\}_{i=1}^G \sim \pi_{\theta_{\mathrm{old}}}(\cdot \mid x)}\left\{\frac{1}{G} \sum_{i=1}^G \min \left[\rho_{i, p}(\theta) \widehat{A}_i, \operatorname{clip}\left(\rho_{i, p}(\theta), 1-\epsilon, 1+\epsilon\right) \widehat{A}_i\right]\right\}.   
\end{equation}
In Lemma~\ref{p=0}, we will show $s_i(\theta)$ in GSPO is equal to $\rho_{i,0}(\theta)$. By a similar discussion for $\rho_{i,p}$, we can obtain the policy gradient 
\begin{align}
\nabla_\theta \mathcal{J}_{\mathrm{H}_s}(\theta)&=\mathbb{E}_{x,\left\{y_i\right\}}\left[\frac{1}{G} \sum_{i=1}^G \mathbb{I}_i(\theta) \cdot \widehat{A}_i \cdot \frac{\rho_{i, p}(\theta)^{1-p}}{\left|y_i\right|} \sum_{t=1}^{\left|y_i\right|} r_{i, t}(\theta)^p \nabla_\theta \log \pi_\theta\left(y_{i, t} \mid x, y_{i,<t}\right)\right], \label{slclippg} \\
\mathbb{I}_i(\theta)&= \begin{cases}0, & \text { if } \widehat{A}_i>0 \text { and } \rho_{i, p}(\theta)>1+\epsilon, \text { or, if } \widehat{A}_i<0 \text { and } \rho_{i, p}(\theta)<1-\epsilon \\ 1, & \text { otherwise. }\end{cases} \nonumber
\end{align}
The unbiased mini-batch estimator is
\begin{equation}\label{seqlevelest}
\widehat{\nabla}_\theta \mathcal{J}_{\mathrm{H}_s}(\theta)=\frac{1}{B} \sum_{b=1}^{B}\frac{1}{G} \sum_{i=1}^G \mathbb{I}_i(\theta) \cdot \widehat{A}_i \cdot \frac{\rho_{i, p}(\theta)^{1-p}}{\left|y_i\right|} \sum_{t=1}^{\left|y_i\right|} r_{i, t}(\theta)^p \nabla_\theta \log \pi_\theta\left(y_{i, t} \mid x, y_{i,<t}\right).
\end{equation}

\subsection{$p=0$ Formulas}\label{F4}
In this section, we extend the three kinds of formulas to $p=0$. By functional analysis, the mean value given by H\"{o}lder $p$-norm for a sequence of positive real numbers $x_1,\dots,x_n$ is 
\[
M_p=\left(\tfrac{1}{n}\right)^{\frac{1}{p}}(x_1^p+\dots+x_n^p)^{\frac{1}{p}}.
\]
The following lemma shows that when $p \rightarrow 0$, the limit of the mean value given by H\"{o}lder $p$-norm is the geometric mean value.
\begin{lemma}\label{p=0}
For any sequence of positive real numbers $x_1, \dots, x_n$, the H\"{o}lder mean $M_p$ converges to the geometric mean as $p$ approaches $0$.
\end{lemma}

\begin{proof}
To evaluate the limit of $M_p$ as $p \to 0$, we first take the natural logarithm of $M_p$:
$$\ln(M_p) = \frac{\ln \left( \frac{1}{n} \sum_{i=1}^n x_i^p \right)}{p}.$$
Let us define an auxiliary function $f(p) = \ln \left( \frac{1}{n} \sum_{i=1}^n x_i^p \right)$. Notice that at $p=0$, $f(0) = \ln \left( \frac{1}{n} \sum_{i=1}^n 1 \right) = \ln(1) = 0$. 

Therefore, the limit as $p \to 0$ is precisely the definition of the derivative of $f(p)$ evaluated at $p=0$:
$$\lim_{p \to 0} \ln(M_p) = \lim_{p \to 0} \frac{f(p) - f(0)}{p - 0} = f'(0).$$

We can explicitly compute the derivative $f'(p)$ using the chain rule:
$$f'(p) = \frac{1}{\frac{1}{n} \sum_{i=1}^n x_i^p} \cdot \left( \frac{1}{n} \sum_{i=1}^n x_i^p \ln(x_i) \right).$$

Evaluating this derivative at $p=0$ gives:
$$f'(0) = \frac{1}{\frac{1}{n} (n)} \cdot \left( \frac{1}{n} \sum_{i=1}^n 1 \cdot \ln(x_i) \right) = \frac{1}{n} \sum_{i=1}^n \ln(x_i).$$

Finally, exponentiating both sides recovers the limit for the original expression $M_p$:
$$\lim_{p \to 0} M_p = \lim_{p \to 0} e^{\ln(M_p)} = e^{f'(0)} = e^{\frac{1}{n} \sum_{i=1}^n \ln(x_i)} = \left( \prod_{i=1}^n x_i \right)^{\frac{1}{n}}.$$
This recovers exactly the geometric mean of the sequence, completing the proof.
\end{proof}
Naturally, we can define all of our objective functions by the geometric mean value for $p=0$. Hence we can see the GSPO (resp. GMPO) objective function is the $p=0$ special case of our sequence-level (resp. token-level) clipping objective function. 

For the policy gradient calculations, we need to discuss the commutativity of operators $\lim_{p\rightarrow0}$ and $\nabla_\theta$. We first define the concept of class $C^1$ multi-variable function $f(p,\theta)$.
\begin{definition}[Class $C^1$]
Let $U \subseteq \mathbb{R} \times \mathbb{R}^d$ be an open set. A function $f: U \to \mathbb{R}$ is said to be jointly continuously differentiable, or of class $C^1$ on $U$ (denoted as $f \in C^1(U)$), if it satisfies that all first-order partial derivatives of $f$, namely $\frac{\partial f}{\partial p}$ and the gradient vector $\nabla_\theta f$, exist at every point $(p, \theta) \in U$ and are jointly continuous on $U$. 
\end{definition}
The next theorem, whose study object is $C^1$-function, can be utilised to guarantee the commutativity of the two operators in the no-clipping case.
\begin{theorem}\label{thm:c1_extension}
Let $f(p, \theta)$ be a parameterised function defined on $(I \setminus \{0\}) \times U$ ($0 \in I$), where $p \in I \subset \mathbb{R}$ and $\theta \in U \subset \mathbb{R}^d$. Suppose the singularity at $p=0$ is removable, such that the extended function defined as
\begin{equation*}
\tilde{f}(p, \theta) = 
\begin{cases} 
f(p, \theta), & \text{if } p \neq 0, \\
\lim_{p \to 0} f(p, \theta), & \text{if } p = 0,
\end{cases}
\end{equation*}
is of class $C^1$ on the joint neighbourhood $I \times U$. Then the differential operator commutes with the limit operator as $p \to 0$:
\begin{equation*}
\lim_{p \to 0} \nabla_\theta \tilde{f}(p, \theta) = \nabla_\theta \tilde{f}(0, \theta) = \nabla_\theta \left( \lim_{p \to 0} f(p, \theta) \right). 
\end{equation*}
\end{theorem}
\begin{proof}
By hypothesis, the extended objective function $\tilde{f}(p, \theta)$ is of class $C^1$ on the joint domain $I \times U$. According to Thm. 9.21 in \cite{rudin1976principles}, the partial derivative operator with respect to $\theta$, denoted as $\nabla_\theta \tilde{f}(p, \theta)$, forms a continuous mapping from $I \times U$ to $\mathbb{R}^d$. Then for any fixed parameter $\theta \in U$, the mapping $p \mapsto \nabla_\theta \tilde{f}(p, \theta)$ is continuous at $p=0$. Thus we obtain the result.   
\end{proof}
For the no-clipping case (\ref{unclip}), the function inside 
the expectation is $L(p, \theta) = \frac{1}{G} \sum_{i=1}^{G} \rho_{i, p}(\theta) \widehat{A}_i$. Because the group size $G$ and the advantage estimates $\widehat{A}_i$ are scalars, the function $L(p, \theta)$ is of class $C^1$ if and only if ${\rho}_{i, p}(\theta)$ is of class $C^1$. This holds true based on the two properties below. 

Firstly, following standard assumptions in deep reinforcement learning, the neural network $\pi_\theta$ utilising smooth activation functions (e.g., Swish, GeLU) and linear transformations is continuously differentiable ($C^1$) with respect to its weights $\theta$. By the chain rule, the strictly positive composite function $r_{i,t}(\theta)$ identically inherits this $C^1$ property. For widely adopted Lipschitz continuous activation functions that are not strictly $C^1$ globally (e.g., ReLU), Rademacher's theorem guarantees that they are differentiable almost everywhere. In the context of stochastic optimisation over continuous parameter spaces, the set of points where the derivative is undefined has Lebesgue measure zero. Consequently, they admit generalised gradients (e.g., Clarke subdifferentials) and are conventionally treated within the $C^1$ framework without loss of theoretical generality.

Secondly, for any $p \neq 0$, $\rho_{i, p}(\theta)$ is a composition of smooth elementary functions and is inherently $C^1$. At the singularity $p=0$, we evaluate the extended function through its logarithmic form:
    \begin{equation}
        \ln \rho_{i, p}(\theta) = \frac{1}{p} \ln \left( \frac{1}{|y_i|} \sum_{t=1}^{|y_i|} e^{p \ln r_{i, t}(\theta)} \right).
    \end{equation}
    By expanding the inner exponential term via its Taylor series around $p=0$, we obtain $\frac{1}{|y_i|} \sum (1 + p \ln r_{i,t} + \mathcal{O}(p^2)) = 1 + p (\frac{1}{|y_i|} \sum \ln r_{i,t}) + \mathcal{O}(p^2)$. Applying the first-order Taylor expansion to the outer logarithm $\ln(1+z) \approx z$ yields:
    \begin{equation}
        \ln \rho_{i, p}(\theta) = \frac{1}{p} \left[ p \left(\frac{1}{|y_i|} \sum_{t=1}^{|y_i|} \ln r_{i,t}(\theta)\right) + \mathcal{O}(p^2) \right].
    \end{equation}
    The parameter $p$ in the denominator perfectly cancels the leading $p$ in the numerator. Because the singularity is analytically removed through this cancellation, the extended function $\rho_{i, p}(\theta)$ and its partial derivatives ($\frac{\partial}{\partial p}$ and $\nabla_\theta$) exhibit no discontinuities or undefined behaviour at $p=0$. 

Therefore, the inner objective function $L(p, \theta)$ is mathematically guaranteed to be jointly $C^1$ on the neighbourhood encompassing $p=0$. This fulfils the prerequisites of Theorem~\ref{thm:c1_extension}, justifying the unconditional exchange of the limit $\lim_{p \to 0}$ and the policy gradient $\nabla_\theta$.

\section{Distribution Deformation}
This appendix supplements Section~\ref{sec:concentration} by providing formal proofs and broader theoretical contexts for our gradient concentration mechanism. Specifically,~\ref{localpro} and~\ref{globalpro} present the proofs for the local weight allocation (Theorem~\ref{thm:amplification}) and the global distributional deformation (Theorem~\ref{thm:deformation}), respectively. Furthermore,~\ref{app:exploration} discusses the profound connection between our three gradient concentration regimes and the traditional exploration-exploitation trade-off in reinforcement learning.  
\label{app:thm1}
\subsection{Local Property}\label{localpro}
In this section, we prove the following theorem, which is mentioned in Section~\ref{sec:concentration} as the local property of the token-level weight allocation $W_{i,t}(p)$ induced by the aggregation parameter $p$.
\begin{theorem}
\label{thm:amplification}
Given an initial parameter state $p_0$. Let $\mathcal{T}^* = \{ t \mid r_{i,t} = \max_k r_{i,k}(\theta) \}$ denote the set of strictly optimal tokens. As $p \to \infty$, $W_{i,t^*}(p) $ increases monotonically and converges to $1/{|\mathcal{T}^*|}$. For any $t \notin \mathcal{T}^*$, there exists a critical $p$-value $p_t > p_0$ such that $W_{i,t}(p)$ reaches its maximum at $p_t$, and strictly decays to zero thereafter as $p \to \infty$.
\end{theorem}
To prove Theorem~\ref{thm:amplification}, we first establish two fundamental lemmas regarding the dynamic weight allocation mechanism controlled by $p$. Let $\mu_{y_i}(p)$ denote the $W_{i,t}$-weighted mean of the log-ratios across the sequence:
\begin{equation}
    \mu_{y_i}(p) := \sum_{t=1}^{|y_i|} W_{i,t}(p) \log r_{i,t}(\theta).
\end{equation}
\begin{lemma}
\label{lem:weight_derivative}
The partial derivative of the gradient weight $W_{i,t}(p)$ with respect to the H\"{o}lder parameter $p$ is strictly governed by its log-ratio relative to the sequence mean $\mu_{y_i}(p)$:
\begin{equation}
    \frac{\partial W_{i,t}(p)}{\partial p} = W_{i,t}(p) \Big( \log r_{i,t}(\theta) - \mu_{y_i}(p) \Big).
\end{equation}
\end{lemma}
\begin{proof}
To provide a complete calculation, we begin by rewriting the gradient weight definition in its exponential form. By expanding the base $r_{i,t}(\theta)^p$, the weight can be expressed as $$W_{i,t}(p) = \frac{\exp(p \log r_{i,t}(\theta))}{\sum_{k=1}^{|y_i|} \exp(p \log r_{i,k}(\theta))}.$$ 
Let $u = \exp(p \log r_{i,t}(\theta))$ and $v = \sum_{k=1}^{|y_i|} \exp(p \log r_{i,k}(\theta))$. Using the chain rule, the derivative of the numerator is simply $$\frac{\partial u}{\partial p} = \exp(p \log r_{i,t}(\theta)) \log r_{i,t}(\theta),$$
while the derivative of the denominator is 
\[
\frac{\partial v}{\partial p} = \sum_{k=1}^{|y_i|} \exp(p \log r_{i,k}(\theta)) \log r_{i,k}(\theta).
\]
The quotient rule formula is
\[
\frac{d}{dp}\left[\frac{u}{v}\right] = \frac{u'v - uv'}{v^2}.
\]
Now, we substitute these components back into the quotient rule formula 
\[
\frac{\partial W_{i,t}(p)}{\partial p} = \frac{\exp(p \log r_{i,t}(\theta)) \log r_{i,t}(\theta)}{\sum_{k} \exp(p \log r_{i,k}(\theta))} - \frac{\exp(p \log r_{i,t}(\theta)) \left[ \sum_{k} \exp(p \log r_{i,k}(\theta)) \log r_{i,k}(\theta) \right]}{\left[ \sum_{k} \exp(p \log r_{i,k}(\theta)) \right]^2}.
\]
Looking closely at the first term, we can isolate the definition of the original weight $W_{i,t}(p)$, leaving us with $W_{i,t}(p) \log r_{i,t}(\theta)$. For the second term, we can factor the fraction into the product of two separate fractions. The first fraction is exactly $W_{i,t}(p)$, and the second fraction represents the weighted sum over all tokens
\[
\frac{\partial W_{i,t}(p)}{\partial p} = W_{i,t}(p) \log r_{i,t}(\theta) - W_{i,t}(p) \left[ \sum_{k=1}^{|y_i|} \frac{\exp(p \log r_{i,k}(\theta))}{\sum_{j} \exp(p \log r_{i,j}(\theta))} \log r_{i,k}(\theta) \right].
\]
We know that the term inside the summation is simply $W_{i,k}(p) \log r_{i,k}(\theta)$, and the entire bracketed sum represents $\mu_{y_i}(p) = \sum_k W_{i,k}(p) \log r_{i,k}(\theta)$. Substituting this notation into our equation gives $$\frac{\partial W_{i,t}(p)}{\partial p} = W_{i,t}(p) \log r_{i,t}(\theta) - W_{i,t}(p) \mu_{y_i}(p)= W_{i,t}(p) \Big( \log r_{i,t}(\theta) - \mu_{y_i}(p) \Big).$$ 
\end{proof}
\begin{lemma}
\label{lem:mean_monotonicity}
Assuming the sequence contains at least two tokens with differing importance ratios, the weighted sequence mean $\mu_{y_i}(p)$ is strictly monotonically increasing with respect to $p$.
\end{lemma}
\begin{proof}
Taking the derivative of $\mu_{y_i}(p)$ with respect to $p$, we have 
\[
    \frac{\partial \mu_{y_i}(p)}{\partial p} = \sum_{t=1}^{|y_i|} \frac{\partial W_{i,t}(p)}{\partial p} \log r_{i,t} = \sum_{t=1}^{|y_i|} W_{i,t}(p) \Big( \log r_{i,t} - \mu_{y_i}(p) \Big) \log r_{i,t}.
\]
Since $\sum_{t=1}^{|y_i|} W_{i,t}(p) = 1$ and by the definition of the mean $\mu_{y_i}(p) = \sum_{t=1}^{|y_i|} W_{i,t}(p) \log r_{i,t}$, we have $$\sum_{t=1}^{|y_i|} W_{i,t}(p) \Big( \log r_{i,t} - \mu_{y_i}(p) \Big) = \mu_{y_i}(p) - \mu_{y_i}(p) = 0.$$
Multiplying this entire zero-sum by the constant $\mu_{y_i}(p)$ yields $$\sum_{t=1}^{|y_i|} W_{i,t}(p) \mu_{y_i}(p) \Big( \log r_{i,t} - \mu_{y_i}(p) \Big) = 0.$$
We can subtract this identically zero term from our derivative equation without changing its value. By grouping the common factor $W_{i,t}(p) \big( \log r_{i,t} - \mu_{y_i}(p) \big)$, the equation collapses into a squared difference
\begin{align*}
\frac{\partial \mu_{y_i}(p)}{\partial p} &= \sum_{t=1}^{|y_i|} W_{i,t}(p) \Big( \log r_{i,t} - \mu_{y_i}(p) \Big) \log r_{i,t} - \sum_{t=1}^{|y_i|} W_{i,t}(p) \mu_{y_i}(p) \Big( \log r_{i,t} - \mu_{y_i}(p) \Big)\\
&= \sum_{t=1}^{|y_i|} W_{i,t}(p) \Big( \log r_{i,t} - \mu_{y_i}(p) \Big)^2. 
\end{align*}
This final expression is exactly the definition of the variance of $\log r_{i,t}$ under the weight distribution $W_i^p$, denoted as $\text{Var}_{W_i^p} (\log r_{i,t})$. Given our assumption that the sequence contains at least two tokens with differing importance ratios, this variance is strictly positive.
\end{proof} 
\begin{proof}[Proof of Theorem~\ref{thm:amplification}]
By definition, $\mathcal{T}^*$ contains the tokens with the strictly maximum importance ratio. Since the weighted sequence mean $\mu_{y_i}(p)$ is a convex combination of all token log-ratios (with $W_{i,k}(p) > 0$ for all finite $p$), it must be strictly bounded by the maximum value: $\mu_{y_i}(p) < \log r_{i,t^*}$ for any finite $p$. By Lemma~\ref{lem:weight_derivative}, the derivative of the weight is governed by its deviation from this mean: $\frac{\partial W_{i,t^*}(p)}{\partial p} = W_{i,t^*}(p) (\log r_{i,t^*} - \mu_{y_i}(p))$. Because both the weight and the deviation are strictly positive, the weight of any optimal token increases monotonically as $p$ grows.

Furthermore, Lemma~\ref{lem:mean_monotonicity} establishes that $\mu_{y_i}(p)$ is strictly monotonically increasing with $p$. Since it is continuously increasing and bounded above by the maximum log-ratio, it is convergent as $p \rightarrow +\infty$. By dividing the numerator and the denominator of $W_{i,t}(p)$ by $r_{i,t^*}^p$, where $r_{i,t^*}$ is the maximum ratio, we rewrite the weight as
\begin{equation*}
    W_{i,t}(p) = \frac{ \left( \frac{r_{i,t}}{r_{i,t^*}} \right)^p }{ \sum_{k=1}^{|y_i|} \left( \frac{r_{i,k}}{r_{i,t^*}} \right)^p }.
\end{equation*}
For any optimal token $t^* \in \mathcal{T}^*$, the base is $1$. For any sub-optimal token $k \notin \mathcal{T}^*$, the base is strictly less than $1$, causing $(r_{i,k}/r_{i,t^*})^p \to 0$ as $p \to \infty$. Consequently, the denominator converges exactly to $|\mathcal{T}^*|$, the total number of optimal tokens. Thus, the weight distribution concentrates entirely on the optimal subset: $\lim_{p \to \infty} W_{i,t^*}(p) = 1/|\mathcal{T}^*|$ and $\lim_{p \to \infty} W_{i,k}(p) = 0$. Since the sequence mean is defined as $\mu_{y_i}(p) = \sum_{t} W_{i,t}(p) \log r_{i,t}$, taking the limit yields:
\begin{equation*}
    \lim_{p \to \infty} \mu_{y_i}(p) = \sum_{t \in \mathcal{T}^*} \left( \frac{1}{|\mathcal{T}^*|} \log r_{i,t^*} \right) + \sum_{k \notin \mathcal{T}^*} \left( 0 \cdot \log r_{i,k} \right) = \log r_{i,t^*}.
\end{equation*}
Combining this exact limit with Lemma~\ref{lem:mean_monotonicity}, we establish that $\mu_{y_i}(p)$ strictly and monotonically approaches the maximum log-ratio $\log r_{i,t^*}$.

For any sub-optimal token $k \notin \mathcal{T}^*$, its log-ratio is strictly less than the maximum ($\log r_{i,k} < \log r_{i,t^*}$). Because $\mu_{y_i}(p)$ continuously sweeps upward towards $\log r_{i,t^*}$, there must exist a critical point $p_t$ where the rising mean exactly crosses the token's log-ratio, yielding $\mu_{y_i}(p_t) = \log r_{i,k}$. For all $p > p_t$, the sequence mean surpasses the token's log-ratio ($\mu_{y_i}(p) > \log r_{i,k}$), which flips the sign of its derivative $\frac{\partial W_{i,k}(p)}{\partial p}$ to negative. Consequently, $W_{i,k}(p)$ reaches its peak at $p_t$ and strictly decays thereafter. As $p \to \infty$, the exponential growth of the optimal tokens' weights strictly dominates the denominator, forcing the weight of all sub-optimal tokens to decay exactly to $0$, and leaving the probability mass uniformly distributed exclusively among the optimal subset with weight $1/|\mathcal{T}^*|$.
\end{proof}

\subsection{Global Property}\label{globalpro}
In this section, we prove the following Theorem~\ref{thm:deformation}, which is mentioned in Section~\ref{sec:concentration} as the global property of the sequence-level distributional deformation induced by the aggregation parameter $p$.

\begin{theorem*}
Assume the sequence $y_i$ contains at least two tokens with distinct importance ratios. Then the Shannon entropy of the weight distribution attains its global maximum at $p = 0$, where $W^0_i = \tfrac{1}{|y_i|} \mathrm{Unif}$, and strictly decreases as $|p|$ increases. Moreover, as $p \to \pm\infty$, $W_i^p$ concentrates uniformly on the subset $\mathcal{T}^{+} = \arg\max_t r_{i,t}(\theta)$ and $\mathcal{T}^{-} = \arg\min_t r_{i,t}(\theta)$, respectively.
\end{theorem*}
\begin{proof}

The Shannon entropy of the weight distribution is defined as 

$$\mathcal{H}(W_i^p) := -\sum_t W_{i,t}(p) \ln W_{i,t}(p).$$ 

To analyse its monotonicity, we compute the derivative of $\mathcal{H}$ with respect to $p$.
First, we compute the derivative of the token weight $W_{i,t}$. Let $$\mathbb{E}_{W_i^p}[\ln r_{i,t}] \coloneqq \sum_{k=1}^{|y_i|} W_{i,k} \ln r_{i,k}$$ denote the expected log-ratio under the current weight distribution. The derivative of the weight is given by
$$ \frac{\partial W_{i,t}}{\partial p} = W_{i,t} \left( \ln r_{i,t} - \sum_k W_{i,k} \ln r_{i,k} \right) = W_{i,t} (\ln r_{i,t} - \mathbb{E}_{W_i^p}[\ln r_{i,t}]).$$
Next, taking the derivative of the entropy yields
$$ \frac{\partial \mathcal{H}}{\partial p} = -\sum_t \left( \frac{\partial W_{i,t}}{\partial p} \ln W_{i,t} + W_{i,t} \frac{\partial \ln W_{i,t}}{\partial p} \right).$$
Notice that 
$$\sum_t W_{i,t} \frac{\partial \ln W_{i,t}}{\partial p} = \sum_t \frac{\partial W_{i,t}}{\partial p} = \frac{\partial}{\partial p}\left(\sum_t W_{i,t}\right) = 0.$$
Thus, the entropy derivative simplifies to $$\frac{\partial \mathcal{H}}{\partial p} = -\sum_t \frac{\partial W_{i,t}}{\partial p} \ln W_{i,t}.$$
To proceed, we explicitly write out $\ln W_{i,t}$. Let $Z(p) \coloneqq \sum_k r_{i,k}^p$.  Since $W_{i,t} = r_{i,t}^p / Z(p)$, taking the natural logarithm gives
$$\ln W_{i,t} = p \ln r_{i,t} - \ln \sum_k r_{i,k}^p = p \ln r_{i,t} - \ln Z(p).$$
Substituting both $\frac{\partial W_{i,t}}{\partial p}$ and $\ln W_{i,t}$ into the simplified entropy derivative, we obtain
$$\frac{\partial \mathcal{H}}{\partial p} = -\sum_t \Big[ W_{i,t} \big( \ln r_{i,t} - \mathbb{E}_{W_i^p}[\ln r_{i,t}] \big) \Big] \Big[ p \ln r_{i,t} - \ln Z(p) \Big].$$
We can expand this product into two separate sums. Notice that the expected deviation from the mean is identically zero. Specifically, because $\mathbb{E}_{W_i^p}[\ln r_{i,t}]$ is a constant with respect to the summation index $t$, we have
$$\sum_t W_{i,t} \big( \ln r_{i,t} - \mathbb{E}_{W_i^p}[\ln r_{i,t}] \big) = \sum_t W_{i,t} \ln r_{i,t} - \mathbb{E}_{W_i^p}[\ln r_{i,t}] \sum_t W_{i,t} =  0.$$
Because this term is zero, any constant multiplier distributed into it vanishes. When we distribute the expanded brackets, the term multiplied by the constant $\ln Z(p)$ completely drops out
\[
\sum_t W_{i,t} \big( \ln r_{i,t} - \mathbb{E}_{W_i^p}[\ln r_{i,t}] \big) \ln Z(p) = 0.
\]
This leaves only the term
$$\frac{\partial \mathcal{H}}{\partial p} = -p \sum_t W_{i,t} \big( \ln r_{i,t} - \mathbb{E}_{W_i^p}[\ln r_{i,t}] \big) \ln r_{i,t}.$$
Finally, we expand the remaining summation by distributing $\ln r_{i,t}$
$$\frac{\partial \mathcal{H}}{\partial p} = -p \left( \sum_t W_{i,t} (\ln r_{i,t})^2 - \mathbb{E}_{W_i^p}[\ln r_{i,t}] \sum_t W_{i,t} \ln r_{i,t} \right).$$
Recognising that $\sum_t W_{i,t} \ln r_{i,t}$ is exactly $\mathbb{E}_{W_i^p}[\ln r_{i,t}]$, this equation collapses into the definition of variance
$$\frac{\partial \mathcal{H}}{\partial p} = -p \left( \mathbb{E}_{W_i^p}[(\ln r_{i,t})^2] - \big( \mathbb{E}_{W_i^p}[\ln r_{i,t}] \big)^2 \right) = -p \cdot \text{Var}_{W_i^p}(\ln r_{i,t}).$$
Since the sequence contains non-uniform importance ratios, the variance is strictly positive ($\text{Var}_{W_i^p}(\ln r_{i,t}) > 0$). Therefore for $p > 0$, $\frac{\partial \mathcal{H}}{\partial p} < 0$, meaning $\mathcal{H}$ strictly decreases.
For $p < 0$, $\frac{\partial \mathcal{H}}{\partial p} > 0$, meaning $\mathcal{H}$ strictly increases towards $p=0$ (or decreases as $p \to -\infty$).
At $p = 0$, $W_i^0$ becomes a uniform distribution where each token is assigned an identical weight of $1/|y_i|$, and $\mathcal{H}$ reaches its global maximum $\ln |y_i|$.

Finally, we evaluate the limits as $p \to \pm\infty$. Let $r_{\max} = \max_t r_{i,t}$ and $\mathcal{M}^* = \{k \mid r_{i,k} = r_{\max}\}$. We can rewrite $W_{i,t}(p)$ by dividing the numerator and denominator by $r_{\max}^p$:
    $$ W_{i,t}(p) = \frac{(r_{i,t}/r_{\max})^p}{\sum_{k \in \mathcal{M}^*} 1 + \sum_{j \notin \mathcal{M}^*} (r_{i,j}/r_{\max})^p}. $$
    For $j \notin \mathcal{M}^*$, $r_{i,j}/r_{\max} < 1$, so $(r_{i,j}/r_{\max})^p \to 0$ as $p \to \infty$. Thus, the denominator converges to $|\mathcal{M}^*|$.
    For the numerator, if $t \in \mathcal{M}^*$, $(r_{i,t}/r_{\max})^p = 1$. If $t \notin \mathcal{M}^*$, $(r_{i,t}/r_{\max})^p \to 0$. Therefore, $\lim_{p \to \infty} W_{i,t}(p) = \frac{1}{|\mathcal{M}^*|}$ for $t \in \mathcal{M}^*$, and $0$ otherwise.   
    
Let $r_{\min} = \min_t r_{i,t}$ and $\mathcal{M}_{\min} = \{k \mid r_{i,k} = r_{\min}\}$. Let $q = -p$, so $q \to \infty$. We can rewrite the weight as:
    $$ W_{i,t}(-q) = \frac{(1/r_{i,t})^q}{\sum_k (1/r_{i,k})^q} = \frac{(r_{\min}/r_{i,t})^q}{\sum_{k \in \mathcal{M}_{\min}} 1 + \sum_{j \notin \mathcal{M}_{\min}} (r_{\min}/r_{i,j})^q}. $$
    Since $r_{\min}/r_{i,j} < 1$ for $j \notin \mathcal{M}_{\min}$, the term $(r_{\min}/r_{i,j})^q \to 0$ as $q \to \infty$. Following the exact same logic, the distribution collapses to a uniform distribution over $\mathcal{M}_{\min}$.
\end{proof}

\subsection{Gradient Concentration vs. Exploration-Exploitation Trade-off}\label{app:exploration}

As established in Section~\ref{sec:concentration}, the aggregation parameter $p$ induces three distinct gradient concentration regimes: upward concentration ($p > 0$) strictly allocates the gradient weights onto high-ratio tokens, uniform dispersion ($p \to 0$) distributes the gradient equally across all tokens, and downward concentration ($p < 0$) upweights the gradient contributions of low-ratio, hesitant tokens. Conceptually, dynamically shifting between these regimes closely mirrors the classical exploration-exploitation dilemma in reinforcement learning \citep{sutton2018reinforcement}. However, the exact nature of this mechanism in the context of LLMs requires careful theoretical contextualisation.

\paragraph{Is a large $p$ considered ``Exploitation''?}
When $p \gg 0$, the algorithm hyper-focuses the gradient updates on tokens where the current policy has already shown the most aggressive improvement relative to the reference policy (i.e., maximal importance ratios). In traditional RL, exploitation implies a behavioural shift---acting greedily according to the current value function during environmental interaction. In our framework, however, a large $p$ acts as a form of post-hoc exploitation that precisely targets two urgent algorithmic crises recently identified in LLM reasoning: the distribution sharpening trap and spurious rewards.

Recent work by \citep{he2025rewarding} reveals that standard GRPO is fundamentally constrained by a distribution sharpening effect, predominantly rewarding tokens that are already likely while failing to amplify sparse, unlikely, yet correct reasoning leaps. Furthermore, \citep{shao2025spurious} demonstrate that Reinforcement Learning with Verifiable Rewards (RLVR) is heavily plagued by spurious signals, where flawed intermediate logic coincidentally yields a correct final answer. When sequence-level advantages are distributed uniformly across the entire trajectory, the optimiser inevitably reinforces this uninformative or even toxic pseudo-logic.

Our upward concentration mechanism ($p > 0$) provides a mathematically principled resolution to both vulnerabilities. It does not exploit by altering the sampling trajectory, but by aggressively filtering the learning signal. By exponentially amplifying the gradient weights of rare, high-ratio tokens, it bypasses the sharpening trap to successfully obtain genuine ``aha moments'' \citep{he2025rewarding}. Simultaneously, by starving the gradient from the bulk of unremarkable tokens, it naturally defends the policy against the integration of spurious background noise \citep{shao2025spurious}. This theoretical intuition is vividly corroborated by our empirical results in Table~\ref{tab:combined_results}: on the AIME24 benchmark, where correct reasoning steps are exceptionally sparse, a highly aggressive static configuration of $p=3$ achieves the peak accuracy of $46.7\%$. Furthermore, as demonstrated in Figure~\ref{fig:traindynamic}, setting $p=+2$ rapidly drives the policy entropy down during the early training stages, visually confirming this intense knowledge-sharpening and mass-concentration effect.

\paragraph{Is a negative $p$ considered ``Exploration''?}
Conversely, when $p < 0$, the gradient concentration shifts toward tokens where the model exhibits hesitation or deviation from previously confident paths. In standard continuous-control RL, exploration is typically enforced via entropy bonuses in the objective \citep{haarnoja2018soft, schulman2017proximal} or noise injection during sampling. While $p < 0$ empirically preserves reasoning diversity, it is not an exploration mechanism in the active sense. Instead, it serves as retrospective diversity preservation. By upweighting less-confident tokens within successful trajectories, it forces the optimiser to consolidate alternative, unconventional reasoning pathways rather than collapsing into a single, greedy solution. We observe this exact dynamic in Figure~\ref{fig:traindynamic}, where a static $p=-2$ sustains significantly higher entropy levels across the entire training trajectory compared to positive $p$ values, actively resisting mode collapse. Moreover, Figure~\ref{fig:logratio} illustrates that decreasing $p$ systematically tightens the gap between the upper and lower envelopes of token-level ratios, redistributing credit to underemphasised tokens. This variance-controlling mechanism proves exceptionally beneficial for dense-signal tasks like MATH500, which strictly favours lower $p$ values (peaking at $p=-1$) to maintain stable optimisation.

To formalise this critical boundary between our gradient mechanisms and traditional RL terminology, we state the following remark.

\begin{remark}
While our concentration mechanism conceptually echoes the exploration-exploitation tradeoff, with $p < 0$ preserving diversity and $p > 0$ sharpening known knowledge, it must not be conflated with classical exploration. In standard RL, exploration refers to actively altering the trajectory sampling distribution (the behavioural policy) to visit unseen states. In contrast, our parameter $p$ operates entirely on the hindsight aggregation of already-sampled trajectories. It functions strictly as a gradient reweighting mechanism, reshaping how the optimisation priority is distributed across a fixed rollout without intervening in how the rollouts are generated.
\end{remark}

\section{Variance Behaviours}
\label{app:thm2}
This appendix provides analysis of the policy gradient variance under the H\"{o}lderPO framework. We first establish a monotonic upper bound for the variance of the unclipped estimator in Section~\ref{E1}, then immediately extend it to the sequence-level clipping case and formalise the structural necessity of sequence-level clipping in Section~\ref{I2}. Subsequently, we derive a more refined variance expression in Section~\ref{orthogonalv} under the assumption of token-gradient orthogonality stated in Section~\ref{or}.

\subsection{An upper bound with monotonicity}\label{E1}
In this section, we prove another version of Theorem~\ref{thm:variance} for the unclipped gradient estimator (\ref{unclibe}).
\begin{theorem}\label{varbound1}
Let $\widehat{\nabla}_\theta \mathcal{J}_{H}$ (Eq.~(\ref{unclibe})) denote the estimator. Assume $\|\nabla_\theta \log \pi_\theta(y_{i,t} \mid x, y_{i,<t})\| \le M$ for all tokens within the batch, the variance admits the bound 
\begin{equation}
\label{eq:var-bound1}
    \|\mathrm{Var}(\widehat{\nabla}_\theta \mathcal{J}_{H})\| \;\le\; \frac{M^2}{B}\, \mathbb{E}\!\left[ \widehat{A}_i^{\,2}\, \rho_{i,p}^2(\theta) \right],
\end{equation}
which is monotonically increasing in $p$ for all $p \in \mathbb{R}$, where $B$ is the batch size.
\end{theorem}
\begin{proof}
We compute the unbiased estimator of the policy gradient by sampling a mini-batch of size $B$, denoted as $\widehat{\nabla}_\theta \mathcal{J}_{\mathrm{H}}$. For a mini-batch containing $B$ i.i.d. sampled trajectories, we have
\[
\widehat{\nabla}_\theta \mathcal{J}_{\mathrm{H}} = \frac{1}{B} \sum_{i=1}^B \hat{g}_i(p),
\]
where the unclipped single-step stochastic gradient $\hat{g}_i(p)$ is
\[
\hat{g}_i(p) = \widehat{A}_i \cdot \nabla_\theta \rho_{i,p}(\theta) = \widehat{A}_i \cdot \left[ \frac{\rho_{i,p}(\theta)^{1-p}}{|y_i|} \sum_{t=1}^{|y_i|} r_{i,t}(\theta)^p \nabla_\theta \log \pi_\theta(y_{i,t} \mid x, y_{i,<t}) \right].
\]
Since every rollout in the mini-batch is independent, we obtain
\[
\text{Var}(\widehat{\nabla}_\theta \mathcal{J}_{\mathrm{H}}) = \text{Var}\left( \frac{1}{B} \sum_{i=1}^B \hat{g}_i(p) \right) = \frac{1}{B^2} \sum_{i=1}^B \text{Var}(\hat{g}_i(p)) = \frac{1}{B} \text{Var}(\hat{g}_1(p)).
\]
For any stochastic gradient,  its variance $\text{Var}(\hat{g}_i(p))$ is strictly bounded by its second moment
\[
\text{Var}(\hat{g}_i(p)) = \mathbb{E}[||\hat{g}_i(p)||^2] - ||\mathbb{E}[\hat{g}_i(p)]||^2 \le \mathbb{E}[||\hat{g}_i(p)||^2].
\]
By applying the triangle inequality and  $||\nabla_\theta \log \pi_\theta(y_{i,t})|| \le M$, we can obtain the upper bound
\begin{align*}
     \lVert \nabla_\theta \rho_{i,p}(\theta) \rVert &=\left\lVert \frac{\rho_{i,p}(\theta)^{1-p}}{|y_i|} \sum_{t=1}^{|y_i|} r_{i,t}(\theta)^p \nabla_\theta \log \pi_\theta(y_{i,t} \mid x, y_{i,<t})\right\rVert\\
    &\le \frac{\rho_{i,p}(\theta)^{1-p}}{|y_i|} \sum_{t=1}^{|y_i|} r_{i,t}(\theta)^p \lVert \nabla_\theta \log \pi_\theta(y_{i,t} \mid x, y_{i,<t}) \rVert \nonumber \\
    &\le M \cdot \frac{\rho_{i,p}(\theta)^{1-p}}{|y_i|} \sum_{t=1}^{|y_i|} r_{i,t}(\theta)^p=  M \cdot \rho_{i,p}(\theta).
\end{align*}
Thus we can bound the squared norm of $\hat{g}_i(p)$ by 
\[
||\hat{g}_i(p)||^2 = \widehat{A}_i^2 \cdot ||\nabla_\theta \rho_{i,p}(\theta)||^2 \le \widehat{A}_i^2 \cdot \left( M \cdot \rho_{i,p}(\theta) \right)^2 = M^2 \cdot \widehat{A}_i^2 \cdot \rho_{i,p}(\theta)^2,
\]
which implies the upper bound of variance of $\widehat{\nabla}_\theta \mathcal{J}_{\mathrm{H}}$
\[
\text{Var}(\widehat{\nabla}_\theta \mathcal{J}_{\mathrm{H}})\le \frac{1}{B}\mathbb{E}[||\hat{g}_i(p)||^2] \le \frac{M^2}{B} \cdot \mathbb{E}_{q, \{y_i\}} \left[ \widehat{A}_i^2 \cdot \rho_{i,p}(\theta)^2 \right].
\]
According to the Generalised Mean Inequality, for any non-uniform sequence of importance ratios, the $p$-mean $\rho_{i,p}(\theta)$ is a strictly monotonically increasing function of $p$. Thus we obtain the conclusion.
\end{proof}

\subsection{Variance and Sequence-level Clipping}\label{I2}
To ensure training stability, policy optimisation algorithms typically employ clipping mechanisms to prevent destructively large updates. In the H\"{o}lderPO framework, we explicitly adopt a \emph{sequence-level clipping} strategy rather than the standard token-level clipping. In this section, we formalise the mathematical necessity of clipping itself, and justify why sequence-level clipping is structurally required to preserve the variance properties established in Theorem~\ref{varbound1}.

\paragraph{Why Clipping is Necessary.}
To understand why the unclipped case is susceptible to exponential explosion, we can analyse its gradient dynamics through an ordinary differential equation. Recall
\begin{equation*}
    \nabla_\theta \rho_{i,p}(\theta) \;=\; \rho_{i,p}(\theta) \sum_{t=1}^{|y_i|} W_{i,t}(p) \nabla_\theta \log \pi_\theta(y_{i,t} \mid x, y_{i,<t}).
\end{equation*}
By abstracting the weighted sum of token-level log-gradients into a single vector $g_w(\theta)$, the gradient equation simplifies to a form 
\[
\nabla_\theta \rho_{i,p}(\theta) = \rho_{i,p}(\theta) g_w(\theta).
\]
During neural network training, the parameters $\theta$ are not static; they continuously evolve along an optimisation trajectory parameterised by a continuous virtual time $\tau$. Applying the chain rule, the rate of change of the ratio with respect to this optimisation time is given by the derivative 
\[
\frac{d \rho_{i,p}}{d\tau} = \big( \nabla_\theta \rho_{i,p}(\theta) \big)^T \frac{d\theta}{d\tau}.
\]
Assuming a standard gradient ascent step aimed at maximising a trajectory with a positive advantage, the parameter update direction $\frac{d\theta}{d\tau}$ naturally aligns with the gradient. Substituting our simplified gradient expression into the derivative yields
\[
\frac{d \rho_{i,p}}{d\tau} = \big( \rho_{i,p} g_w(\theta) \big)^T \frac{d\theta}{d\tau} = \rho_{i,p} \left( g_w(\theta)^T \frac{d\theta}{d\tau} \right).
\]
Because the optimiser attempts to increase the likelihood of these correct tokens, the update direction $\frac{d\theta}{d\tau}$ forms an acute angle with the composite gradient direction $g_w(\theta)$. This implies that their inner product is a strictly positive scalar, which we can denote as $k(\tau) > 0$. Substituting this scalar reduces the complex optimisation dynamics into a canonical ODE for exponential growth 
\[
\frac{d \rho_{i,p}}{d\tau} = k(\tau) \rho_{i,p}.
\]
Integrating this differential equation over the time interval $[0, \tau]$ provides the exact analytical solution
\[
\rho_{i,p}(\tau) = \rho_{i,p}(0) \exp\left( \int_0^\tau k(t) \, dt \right).
\]
This mathematically dictates that without an explicit clipping mechanism to interrupt the ODE, the scaling factor will inevitably suffer from an unbounded exponential explosion. Therefore, explicitly bounding the surrogate ratio via a clipping mechanism is an absolute prerequisite.

\begin{proof}[Proof of Theorem~\ref{thm:variance}]
The clipping operator acts as a binary sequence-level mask $\mathbb{I}_i(\theta) \in \{0, 1\}$ applied directly to the aggregated ratio $\rho_{i,p}(\theta)$ and its advantage $\widehat{A}_i$ (see Eq.~(\ref{slclippg})). Consequently, the clipped stochastic gradient $\hat{g}_i^{\text{clip}}(p)$ is either preserved in full (when $\mathbb{I}_i = 1$) or completely zeroed out (when $\mathbb{I}_i = 0$). Mathematically, this guarantees that the squared norm of the clipped gradient is universally bounded by the unclipped one:
\begin{equation*}
    \|\hat{g}_i^{\text{clip}}(p)\|^2 \;=\; \mathbb{I}_i(\theta) \cdot \|\hat{g}_i(p)\|^2 \;\le\; \|\hat{g}_i(p)\|^2.
\end{equation*}
Because the variance is bounded by the second moment, the upper bound we derived in Theorem~\ref{varbound1} carries over:
\begin{equation*}
    \|\mathrm{Var}(\widehat{\nabla}_\theta \mathcal{J}_{H_s})\| \;\le\; \frac{1}{B}\mathbb{E}\big[\|\hat{g}_i^{\text{clip}}(p)\|^2\big] \;\le\; \frac{M^2}{B}\, \mathbb{E}\!\left[ \widehat{A}_i^{\,2}\, \rho_{i,p}^2(\theta) \right].
\end{equation*}
Thus, the monotonic relationship between the variance bound and the parameter $p$ remains intact.  
\end{proof}
In contrast, \emph{token-level clipping} applies the clipping operator \emph{inside} the summation over individual tokens. It unpredictably alters specific token ratios, destroying the correspondence between the outer multiplier $\rho_{i,p}(\theta)^{1-p}$ and the inner weights $W_{i,t}(p)$ (see Eq.~(\ref{tles})). This structural fracture voids the monotonic upper bound derived above, making the variance highly uncontrolled. Our empirical results in Appendix~\ref{app:token_clip} (Table~\ref{tab:token_overall_results}) corroborate this: token-level clipping narrows the performance spread across $p$, confirming that the parameter $p$ loses its tight, predictable control over gradient variance.

\subsection{Approximate orthogonality of policy gradients}\label{or}
\begin{assumption}\label{orthogonal}
In long-horizon reasoning tasks, we assume that within a given sequence $y_i$, the policy gradients with respect to tokens at different positions are approximately orthogonal, i.e. $\mathbb{E}[g_t^T g_k] \approx 0$ ($g_t = \nabla_\theta \log \pi_\theta(y_{i,t} \mid x, y_{i,<t})$) for any two distinct tokens $t \neq k$ in $y_i$.    
\end{assumption}
This assumption is practical and well-founded in the context of LLMs due to two factors: the blessing of dimensionality and linguistic feature decoupling. First, geometrically, in a parameter space with billions of dimensions, any two distinct gradient vectors are statistically bound to be nearly orthogonal \citep{vershynin2018high}. Second, from a linguistic and mechanistic perspective, tokens at different positions within a long sequence typically serve distinct semantic and syntactic functions (e.g., predicting a generic preposition versus a complex domain-specific entity). Recent advances in mechanistic interpretability reveal that Transformer feed-forward layers operate as sparse key-value memories, where distinct neurons exclusively fire for specific linguistic patterns \citep{geva2020transformer, bricken2023towards}. Consequently, the subset of parameters responsible for encoding and predicting these distinct tokens are largely disjoint. This functional specialisation ensures that the back-propagated learning signals for distinct tokens are routed to different parameter subspaces, naturally leading to approximately uncorrelated, orthogonal gradient directions.

\subsection{Monotonicity of Variance}\label{orthogonalv}

While Assumption~\ref{orthogonal} and Remark~\ref{r2} establish that the token-level gradients are approximately orthogonal and uniformly bounded in practice, analysing the exact variance dynamics requires a formal mathematical model. To achieve this, we adopt a standard theoretical abstraction: we transition from the empirical approximations to an idealised setting where these conditions hold exactly. This formal idealisation allows us to decouple the intrinsic sequence-level aggregation behaviour from token-specific optimisation noise, paving the way for the analysis presented in Theorem~\ref{variance_optimal}.

\begin{theorem}\label{variance_optimal}
Under the idealised assumption of exact token-level gradient orthogonality ($\mathbb{E}[g_t^T g_k] = 0$ for $t \neq k$) and a uniformly bounded expected gradient norm ($\mathbb{E}[||g_t||^2] = M^2$), the exact second moment (and proportionally, the variance) of the H\"{o}lder-aggregated policy gradient estimator $\hat{g}_i$ simplifies to:
$$ \mathbb{E}[||\hat{g}_i||^2] = \widehat{A}_i^2 M^2 \rho_{i,p}(\theta)^2 \sum_{t=1}^{|y_i|} W_{i,t}(p)^2. $$
Consequently, we have the following properties:
\begin{enumerate}
    \item As $p$ decays from $+\infty$ to $0$, the variance strictly decreases.
    \item As $p \to -\infty$, the weight concentration index (HHI), defined as $\sum_{t} W_{i,t}(p)^2$, grows exponentially and collapses to $1$, counteracting the decrease in the H\"{o}lder mean $\rho_{i,p}(\theta)^2$.
    \item There exists a $p^* \le 0$ that strictly minimises the variance.
\end{enumerate}       
\end{theorem}

\begin{proof}[Proof of Theorem~\ref{variance_optimal}]
Keeping symbols from the proof of Theorem~\ref{varbound1} and Assumption~\ref{orthogonal}, the H\"{o}lder-aggregated policy gradient for a single trajectory $y_i$ is:
\begin{equation*}
    \hat{g}_i(p) = \widehat{A}_i \rho_{i,p}(\theta) \sum_{t=1}^{|y_i|} W_{i,t}(p) g_t.
\end{equation*}
We expand the squared $L_2$-norm of this estimator:
\begin{equation*}
    ||\hat{g}_i(p)||^2 = \widehat{A}_i^2 \rho_{i,p}(\theta)^2 \left( \sum_{t=1}^{|y_i|} \sum_{k=1}^{|y_i|} W_{i,t}(p) W_{i,k}(p) g_t^T g_k \right).
\end{equation*}
Taking the expectation with respect to the trajectory distribution, and applying the idealised token-level gradient orthogonality ($\mathbb{E}[g_t^T g_k] = 0$ for $t \neq k$), all cross-terms vanish exactly. Using the idealised uniform expected magnitude assumption ($\mathbb{E}[||g_t||^2] = M^2$), we obtain the exact second moment:
\begin{equation*}
    \mathbb{E}[||\hat{g}_i(p)||^2] = \widehat{A}_i^2 \rho_{i,p}(\theta)^2 \sum_{t=1}^{|y_i|} W_{i,t}(p)^2 \mathbb{E}[||g_t||^2] = \widehat{A}_i^2 M^2 \rho_{i,p}(\theta)^2 \sum_{t=1}^{|y_i|} W_{i,t}(p)^2.
\end{equation*}
Recognising that $\sum_{t} W_{i,t}(p)^2$ is exactly the Herfindahl-Hirschman Index (HHI) of the weight distribution, denoted as $\mathcal{H}_{HHI}(p)$, we analyse the exact variance dynamics based on this factorisation $V(p) \propto \rho_{i,p}(\theta)^2 \cdot \mathcal{H}_{HHI}(p)$:

\textbf{Proof of Property 1.} 
As established in Lemma~1, the H\"{o}lder mean $\rho_{i,p}(\theta)$ is strictly monotonically increasing with respect to $p$. Concurrently, for $p > 0$, the weight distribution gradually disperses from a strict one-hot distribution (at $p \to +\infty$) towards a uniform distribution (at $p \to 0$). Because the uniform distribution globally minimises the HHI (where $\mathcal{H}_{HHI}(0) = 1/|y_i|$), $\mathcal{H}_{HHI}(p)$ is strictly monotonically decreasing as $p$ decays from $+\infty$ to $0$. Since both $\rho_{i,p}(\theta)^2$ and $\mathcal{H}_{HHI}(p)$ are strictly decreasing as $p$ decreases in $(0, +\infty)$, their product $V(p)$ must strictly decrease.

\textbf{Proof of Property 2.} 
As $p \to -\infty$, the H\"{o}lder mechanism heavily upweights the minimum elements. The weight distribution collapses into a one-hot distribution centred exclusively on the token(s) with the minimum importance ratio. Consequently, $\lim_{p \to -\infty} W_{i,t_{\min}}(p) = 1$, which drives the concentration index $\mathcal{H}_{HHI}(p)$ to grow exponentially back to its maximum possible value of $1$. This sharp exponential growth of the HHI acts as a strong regulariser, counteracting the continuing decay of the H\"{o}lder mean $\rho_{i,p}(\theta)^2$.

\textbf{Proof of Property 3.} 
Let $V(p) = \rho_{i,p}^2 \cdot \mathcal{H}_{HHI}(p)$ represent the variance objective. From Property~1, $V(p)$ is strictly monotonically increasing for all $p \in (0, +\infty)$, implying that the global minimum of $V(p)$ cannot exist in the positive domain. At $p = 0$, the variance evaluates to $V(0) = \rho_{i,0}^2 \cdot (1/|y_i|)$. As $p$ decreases into the negative domain ($p < 0$), $\rho_{i,p}^2$ continues to decay, but $\mathcal{H}_{HHI}(p)$ begins to increase towards $1$ (Property~2). Because $V(p)$ is a continuous function bounded from below (by $0$) defined on the closed interval $[-\infty, 0]$, by the Extreme Value Theorem, it must attain a minimum. Since it strictly increases for $p > 0$, this global variance-minimising point $p^*$ is strictly guaranteed to satisfy $p^* \le 0$. 
\end{proof}

\begin{remark}\label{r2}
The assumption that $\mathbb{E}[||g_t||^2] \approx M^2$ (i.e., token-level policy gradients have homogeneously expected magnitudes) is both a standard simplification and practically well-founded for modern LLMs for two reasons.

1.~\textit{Architectural Normalisation:} Modern LLMs heavily rely on RMSNorm or LayerNorm before the final classification head. This strictly bounds the magnitude of the hidden states, thereby stabilising the scale of the back-propagated log-probability gradients across different token positions.

2.~\textit{Statistical Stationarity over Long Horizons:} While specific tokens might incur momentary gradient spikes, the expected squared norm over the data distribution tends toward a stationary value $M^2$ because all tokens share the same underlying language modelling head and projection matrices.

\end{remark}

\section{Quantitative Advantage of Dynamic Scheduling}
\label{app:thm3}

\begin{remark}
In Theorem~\ref{thm:dynamic}, the exponential amplification of the sparse reward signal relies on the \emph{pre-saturation condition} $r_{i,t^{*}}^{p_{\text{high}}} \ll n - 1$. This inequality is not merely a mathematical artefact, but rather a formalisation of the early-phase training dynamics in long-horizon LLM reasoning. We clarify its physical meaning and empirical validity as follows.

1. Mathematically, the term $r_{i,t^{*}}^{p_{\text{high}}}$ represents the amplified signal of the single correct reasoning token, while $n - 1$ represents the aggregated background mass of the remaining unremarkable tokens in the sequence. When $r_{i,t^{*}}^{p_{\text{high}}} \gg n - 1$, the weight $W_{i,t^{*}}$ saturates to $1$, meaning the model has already become overwhelmingly confident in this step, and the gradient is completely monopolised. Therefore, the pre-saturation condition ($\ll n - 1$) defines the critical case where the policy has discovered a high-reward token but is not yet absolutely confident. It is precisely in this window that the model desperately needs the exponential gradient boost provided by $p_{\text{high}}$ to escape the noise.

2. This condition is exceptionally easy to satisfy in modern LLM reasoning tasks (e.g., AIME or MATH) due to two structural factors:
\begin{itemize}
    \item \textbf{Massive Sequence Length ($n$):} Chain-of-Thought (CoT) trajectories are inherently long, often spanning hundreds or thousands of tokens ($n \sim 10^3$). Consequently, the background mass $n-1$ provides a massive buffer.
    \item \textbf{Early-Phase Low Confidence ($r_{i,t^{*}}$):} In the early stages of RLVR, finding the correct reasoning path is a rare event. Even when the model stumbles upon the correct logic, its generation probability $\pi_\theta$ is only marginally higher than the reference $\pi_{\text{ref}}$. Thus, the initial ratio $r_{i,t^{*}}$ is moderately greater than $1$, but absolutely not large enough to let its $p$-th power immediately overpower thousands of background tokens.
\end{itemize}

3. Crucially, the pre-saturation condition justifies our dynamic scheduling design. As training progresses, the model fits the correct trajectory, and $r_{i,t^{*}}$ grows. Eventually, the condition $r_{i,t^{*}}^{p_{\text{high}}} \ll n - 1$ will be violated (saturation occurs), rendering $p_{\text{high}}$ mathematically ineffective at further isolating the signal. Exactly at this point, our dynamic schedule seamlessly decays $p \to p_{\text{low}} \le 0$, shifting the algorithmic focus from signal amplification to variance contraction (Theorem~\ref{thm:dynamic}, Part~2).
\end{remark}

\begin{proof}[Proof of Theorem~\ref{thm:dynamic}]

Let $R \coloneqq r_{i,t^{*}} \gg 1$. For the remaining tokens $t \neq t^{*}$, since their ratios are constant-bounded, we can denote their sum of $p$-th powers as $S(p) \coloneqq \sum_{t \neq t^{*}} r_{i,t}^p = \Theta(n - 1)$, which holds uniformly for any $p$ in a bounded interval $[p_{\text{low}}, p_{\text{high}}]$. By definition, the weight of the high-ratio token is:
\begin{equation*}
    W_{i,t^{*}}(p) \;=\; \frac{R^p}{R^p + S(p)}.
\end{equation*}
Therefore, the relative amplification of the gradient weight when shifting from $p_{\text{stat}}$ to $p_{\text{high}}$ is given by:
\begin{align*}
    \frac{W_{i,t^{*}}(p_{\text{high}})}{W_{i,t^{*}}(p_{\text{stat}})} 
    \;&=\; \frac{R^{p_{\text{high}}}}{R^{p_{\text{high}}} + S(p_{\text{high}})} \cdot \frac{R^{p_{\text{stat}}} + S(p_{\text{stat}})}{R^{p_{\text{stat}}}} \\[6pt]
    \;&=\; R^{\,p_{\text{high}} - p_{\text{stat}}} \cdot \frac{R^{p_{\text{stat}}} + S(p_{\text{stat}})}{R^{p_{\text{high}}} + S(p_{\text{high}})}.
\end{align*}
Under the pre-saturation condition $R^{p_{\text{high}}} \ll n - 1$, the term $R^{p_{\text{high}}}$ is asymptotically dominated by the denominator's background sum $S(p_{\text{high}}) = \Theta(n-1)$. Since $p_{\text{stat}} < p_{\text{high}}$, we naturally also have $R^{p_{\text{stat}}} \ll n - 1$. Consequently, the fractional multiplier is bounded from below by a strictly positive constant $C = \Theta(1)$:
\begin{equation*}
    \frac{R^{p_{\text{stat}}} + S(p_{\text{stat}})}{R^{p_{\text{high}}} + S(p_{\text{high}})} \;\ge\; \frac{S(p_{\text{stat}})}{R^{p_{\text{high}}} + S(p_{\text{high}})} \;\ge\; C \;>\; 0.
\end{equation*}
Substituting $R = r_{i,t^{*}}$ back into the expression yields the desired exponential lower bound for the signal amplification:
\begin{equation*}
    \frac{W_{i,t^{*}}(p_{\text{high}})}{W_{i,t^{*}}(p_{\text{stat}})} \;\ge\; C \cdot r_{i,t^{*}}^{\,p_{\text{high}} - p_{\text{stat}}}.
\end{equation*}

By the definition provided in the theorem, the variance bound term is exactly $V(p) \coloneqq \mathbb{E}[\widehat{A}_i^{\,2}\, \rho_{i,p}^2(\theta)]$. Assuming the importance ratios within the sequence are non-degenerate (i.e., not all tokens share the exact same ratio), the generalised mean inequality guarantees that the H\"{o}lder mean $\rho_{i,p}(\theta)$ is strictly monotonically increasing with respect to $p$. Thus, for any $p_{\text{low}} < p_{\text{stat}}$, we have $\rho_{i,p_{\text{low}}}(\theta) < \rho_{i,p_{\text{stat}}}(\theta)$ pointwise for every sequence $y_i$. Since the squared advantage $\widehat{A}_i^{\,2} \ge 0$ (and is strictly positive for meaningful updates), squaring the strictly positive H\"{o}lder means yields the following pointwise inequality for the random variables:
\begin{equation*}
    \widehat{A}_i^{\,2}\, \rho_{i,p_{\text{low}}}^2(\theta) \;<\; \widehat{A}_i^{\,2}\, \rho_{i,p_{\text{stat}}}^2(\theta).
\end{equation*}
Taking the expectation over the trajectory distribution strictly preserves this inequality, yielding:
\begin{equation*}
    \mathbb{E}\!\left[\widehat{A}_i^{\,2}\, \rho_{i,p_{\text{low}}}^2(\theta)\right] \;<\; \mathbb{E}\!\left[\widehat{A}_i^{\,2}\, \rho_{i,p_{\text{stat}}}^2(\theta)\right],
\end{equation*}
which directly concludes that $V(p_{\text{low}}) < V(p_{\text{stat}})$.
\end{proof}

\section{Broader Impacts}
\label{broader impacts}

By improving the efficiency and stability of RL post-training, H\"{o}lderPO can reduce the compute required to reach competitive performance on complex reasoning benchmarks, lowering the barrier for researchers and practitioners to develop capable reasoning models. Like any policy optimisation method, it inherits the standard dual-use risks of strong LLMs, including potential misuse for misinformation or automated content generation. A concern more specific to our framework is that the gradient amplification in the positive-pp
p regime can intensify reward hacking when learning signals are misspecified, a limitation we discuss explicitly in Section~\ref{sec:conclusion}.
    
\FloatBarrier



\end{document}